\pdfoutput=1
\documentclass{article}

\usepackage{iclr2027_conference,times}

\newcommand{\sketchqed}{\hfill$\blacksquare$}

\usepackage{bbm}
\usepackage{graphicx}
\usepackage{enumitem}
\usepackage{amsmath}
\usepackage{amssymb}
\usepackage{amsthm}

\newcounter{acond}
\renewcommand{\theacond}{A\arabic{acond}}

\newcommand{\Aconditem}{%
  \refstepcounter{acond}\item[\textbf{(\theacond)}]%
}

\theoremstyle{plain}
\newtheorem{theorem}{Theorem}[section]
\newtheorem{lemma}[theorem]{Lemma}
\newtheorem{corollary}[theorem]{Corollary}
\newtheorem{proposition}[theorem]{Proposition}

\theoremstyle{definition}
\newtheorem{definition}[theorem]{Definition}
\newtheorem{example}[theorem]{Example}

\theoremstyle{remark}
\newtheorem{remark}[theorem]{Remark}

\newcommand{\Pcal}{\mathcal{P}}
\newcommand{\Hcal}{\mathcal{H}}
\newcommand{\Wcal}{\mathcal{W}}
\newcommand{\Zcal}{\mathcal{Z}}

\DeclareMathOperator*{\argmin}{arg\,min}
\newcommand{\var}{\mathrm{Var}}

\usepackage[utf8]{inputenc}
\usepackage[T1]{fontenc}
\usepackage[colorlinks=true, linkcolor=blue, citecolor=blue]{hyperref}
\usepackage{url}
\usepackage{booktabs}
\makeatletter
\g@addto@macro\normalsize{%
  \setlength{\abovedisplayskip}{6pt plus 2pt minus 2pt}%
  \setlength{\belowdisplayskip}{6pt plus 2pt minus 2pt}%
  \setlength{\abovedisplayshortskip}{2pt plus 1pt minus 1pt}%
  \setlength{\belowdisplayshortskip}{4pt plus 1pt minus 1pt}%
}
\makeatother
\usepackage{amsfonts}
\usepackage{nicefrac}
\usepackage{microtype}
\usepackage{xcolor}
\newcommand{\rev}[1]{{#1}}

\title{Fast Rates for Nonstationary Weighted Risk Minimization}

\author{%
  Tobias Brock \\
  LMU Munich \\
  Munich Center for Machine Learning
  \And
  Thomas Nagler \\
  LMU Munich \\
  Munich Center for Machine Learning
}

\iclrfinalcopy

\begin{document}

\maketitle
\lhead{Preprint}

\begin{abstract}
Weighted empirical risk minimization is a common approach to prediction under distribution drift.
This article studies its out-of-sample prediction error under nonstationarity.
We provide a general decomposition of the excess risk into a learning term and an error term associated with distribution drift, and prove oracle inequalities for the learning error under mixing conditions.
The learning bound holds uniformly over arbitrary weight classes and accounts for the effective sample size induced by the weight vector, the complexity of the weight and hypothesis classes, and potential data dependence.
We also analyze the drift error in the square loss setting and provide explicit bounds for structural-break and smoothly drifting processes.
We then combine the learning and drift analysis in (auto-) regression problems with linear models, basis approximations,
and neural networks. In all cases, the optimal weights are characterized by a bias--variance trade-off between learning and drift, while minimax-optimal rates (up to logarithmic factors) are recovered when specialized to unweighted and stationary settings.
\end{abstract}

\section{Introduction}\label{sec:intro}

The study of empirical risk minimizers (ERMs) is central to learning theory. In particular, we are interested in the following question: if we learn a hypothesis $\hat{h} \in \mathcal{H}$ from samples $X_1,\dots,X_n$, how well does it perform on an unseen observation $X_{n+1}$?
The setting where samples are independent and identically distributed ($iid$) has been extensively studied.
However, real-world data is often nonstationary and dependent, e.g., financial time series \citep{cont2001empirical, tsay2010analysis} or weather data \citep{gneiting2007probabilistic, wilks2011statistical}.
A common approach to handling nonstationarity is to introduce weights $w \in \mathbb R^n$ to the empirical risk: for some loss function $L$, define
\begin{align*}
   R_n^w(h) = \sum_{t=1}^n w_t L(X_{t}, h), \qquad \hat h_w = \argmin_{h \in \mathcal{H}} R_n^w(h).
\end{align*}
For example, if the distributions are drifting slowly, one may improve the performance of the minimizer by giving greater weight to more recent observations; see, e.g., \citet{VogelEtAl2020WERM, mazzetto2023adaptivealgorithmlearningunknown, temporalshifts}.
However, when weights concentrate too much on a few observations, the \emph{effective sample size}   $n_{\text{eff}}(w) = 1/\|w\|_2^2$ \citep{kish1992weighting} is small and renders the problem noisier, a trade-off that needs to be accounted for.

We want to find $\hat h_w$ that is close to the Bayes optimal model $h_{P_{n + 1}}^* = \argmin_h \mathbb{E}[L(X_{n + 1}, h)]$.
It is thus natural to decompose the excess risk into a learnable and a deterministic component:
  \begin{align*}
     \underbrace{\mathbb{E}[L(X_{n+1}, \hat h_{w})] - \mathbb{E}[L(X_{n+1}, h_{P_{n + 1}}^*)]}_{\text{excess out-of-sample risk}}
    &\lesssim \underbrace{\| \hat h_{w} - h_{w}^* \|_{L_2(P_{n+1})}^2}_{\text{learning error}} +  \underbrace{\|h_{w}^* -  h_{P_{n + 1}}^*\|_{L_2(P_{n+1})}^2}_{\text{drift error}},
\end{align*}
where $ h_w^* = \argmin_h \sum_{t = 1}^n w_t \mathbb{E}[L(X_{t}, h)]$, with both minima taken over all measurable functions.

This decomposition makes the trade-off between learning and drift explicit: larger $n_{\text{eff}}(w)$ can reduce the learning error but may increase drift if the distributions evolve quickly over time.
Moreover, the drift error is solely affected by the nonstationarity of the data and the choice of the weights, making it independent of the hypothesis class $\Hcal$.
In this work, we study the out-of-sample excess risk of weighted ERM under nonstationarity by analyzing the learning and drift errors.

\subsection*{Contributions}

\paragraph{Fast rates for the learning error.}
In Section \ref{sec:main}, we state our main result (Theorem \ref{thm:fast}), an oracle inequality for weighted ERM: for any weight class $\mathcal{W}$, with high probability,
\begin{align*}
        \forall w \in \mathcal{W}\colon \quad \|\hat{h}_w - h^*_{w} \|_{L_2(P_{n + 1})} \lesssim r(\|w\|_2),
\end{align*}
for a rate function $r$ determined by the dependence in the data and the complexities of weight and hypothesis classes, explicitly depending on the effective sample size. In particular:
\begin{itemize}[leftmargin=1.5em,itemsep=4pt,topsep=0pt,parsep=0pt,partopsep=0pt]
\item \emph{Uniformity over weights:} the guarantee is simultaneous over a generic class $\mathcal W$, including any data-dependent choice $\widehat w\in\mathcal W$.
\item \emph{Weight-adaptive models:} the hypothesis class $\mathcal H_w$ may depend on the selected weights.
\item \emph{Sharp dependence control:} controlling local empirical-risk deviations allows an iid-order leading rate under polynomial mixing, thus improving on existing results even for stationary data and uniform weights.
\end{itemize}

Section~\ref{sec:lower-bounds-main} gives a matching lower bound for mean estimation that recovers the effective sample size and complexity terms.

\paragraph{Bounds on the drift error.}
In Section~\ref{sec:driftsquare}, we analyze the drift error for structural break processes and smoothly drifting marginals.
For square loss, the drift error depends on how the selected weights interact with changes in the regression function, and we provide explicit bounds for both settings.

\paragraph{Excess-risk analysis for selected models.}
Section~\ref{sec:applications} combines learning and drift for linear models, basis expansions, and neural networks.
We recover the corresponding iid minimax rates up to logarithmic factors and show how nonstationarity affects the excess risk in the structural break setting:
the optimal weight choice must balance the learning and drift errors.
We also give a numerical illustration that demonstrates the fast rate under polynomial mixing and the learning-drift trade-off.

\subsection*{Related work}

A large body of literature studies learning under nonstationarity by quantifying distribution drift via the discrepancy, which captures the worst-case expected loss between two time points with respect to the hypothesis class \citep{mansour2009domain}. Generalization results leveraging the discrepancy have been derived for independent drifting processes \citep{NIPS2011_1e1d1841, MohriMunozMedina2012, pmlr-v206-awasthi23b, mazzetto2023adaptivealgorithmlearningunknown}. Such bounds either require restricting the drift rate or require that the discrepancy is sufficiently small over some specified look-back window to be meaningful. Alternatively, \citet{ol} propose selecting an adaptive look-back window in a nonstationary environment to obtain optimal regret guarantees.

Other authors consider more general settings by jointly relaxing independence and stationarity. \citet{Kuznetsov2017} and \citet{hanneke} derive generalization results for nonstationary mixing processes using the discrepancy; \citet{KuznetsovMohri2016, kuznetsov2020discrepancy} impose no restriction on the dependence structure.
The latter works also consider weighted ERM, with weights learned by minimizing a discrepancy term.
Our decomposition instead isolates drift from the learning problem.

For stationary data, \citet{SteinwartChristmann2009} and \citet{HangSteinwart2014} derive oracle inequalities and fast learning rates for regularized empirical risk minimization for stationary data under mixing assumptions, and \citet{HangSteinwart2017} prove a Bernstein-type inequality for geometrically mixing processes and demonstrate learning consequences. \citet{FarahmandSzepesvari2012} study regularized least-squares regression under exponentially $\beta$-mixing dependence and \citet{BarreraGobet2021} provide generalization bounds for nonparametric regression under $\beta$-mixing, while additionally relaxing the stationarity assumption. \citet{DebMukherjee2024} develop an empirical-process framework for stationary $\beta$- and $\rho$-mixing data and derive localized bounds that imply fast rates in a range of nonparametric problems.
Similarly, \citet{Alquier2013} derive PAC-Bayesian oracle inequalities for forecasting and establish fast rates under restrictive dimensionality and dependence conditions.
Importantly, \citet{FarahmandSzepesvari2012, HangSteinwart2014, BarreraGobet2021, DebMukherjee2024} obtain rates in stationary or unweighted settings that are polynomially worse than $1/n$ for polynomially mixing data, which is not the case for our results.

\section{Setup and notation}\label{sec:setup}

Let $X_1, \dots, X_{t + 1}$ be a sequence of random variables with marginal laws $ P_1, \dots, P_{t + 1}$ and arbitrary dependence.
We consider a prediction problem where, at time $t$, we want to learn a hypothesis $h$ from $X_1, \dots, X_t$ so that the expected out-of-sample loss $\mathbb{E}[L(X_{t+1}, h)]$ is small. This is hopeless without further assumptions, as $P_{t+1}$ can be completely unrelated to $P_1, \dots, P_t$.
In most practical scenarios, however, the distributions $P_t$ evolve gradually or have infrequent regime changes.
A common approach is to exploit this by considering a weighted empirical risk minimization (ERM) procedure where recent losses are weighted higher.

To be more precise, denote the weighted sample average and its expectation by
\begin{align*}
    R_n^{w}(h) = \sum_{t=1}^nw_t L(X_t,h), \qquad R^{w}(h)=\sum_{t=1}^n w_t \mathbb{E}[L(X_{t},h)],
\end{align*}
where $w\in \mathbb{R}^n$ is a weight vector with $\sum_{t=1}^n w_t=1$. For every such weight vector, write $P_w=\sum_{t=1}^n w_tP_t$.
Define the weighted empirical risk minimizer and optimal model as
\begin{align*}
    \hat{h}_w = \argmin_{h \in \mathcal{H}_w} R_n^w(h), \qquad h^*_w = \argmin_{h \text{ measurable}} R^w(h),
\end{align*}
where $\mathcal{H}_w$ is a hypothesis class. We intentionally let $\mathcal{H}_w$ depend on $w$ since we should pick the complexity of the hypothesis class according to the effective sample size $n_{\text{eff}}(w) = 1/\|w\|_2^2$.
Moreover, it is unclear which weighting scheme is best for a given problem. The weights are typically themselves selected, for example, through backtesting, so we want to derive bounds that hold uniformly over a class of weights $\mathcal{W}$.

\begin{example}\label{exp:weights}
    A few commonly used weight families  \citep{GARDNER2006637, hyndman2008forecasting} are:
    \begin{enumerate}[label=(\roman*),leftmargin=3em,itemsep=4pt,topsep=0pt,parsep=0pt,partopsep=0pt]
        \item Uniform window:
$\displaystyle w_i^{(t)} = \mathbbm{1}\{i \leq t\} \times
\begin{cases}
1/s, & i\in\{t-s+1,\dots,t\},\\
0, & i\in\{1,\dots,t-s\}.
\end{cases}$, $s \in \{1,\dots, t\}$.
    \item   Exponential smoothing: $\displaystyle w_{i}^{(t)} = \mathbbm{1}\{i \leq t\}  \frac{ \exp(-\theta(t-i))}{\sum_{j=1}^t \exp(-\theta(t-j))}$, $\theta \in (0, \infty)$.
    \end{enumerate}
The dependence on $t$ is made explicit to emphasize that the weight vectors change over time and only assign non-zero weights on past observations.
Our analysis also allows for negative weights, e.g.,
Brown double exponential weights, discussed in Appendix~\ref{app:brown-weights}, provided $P_w$ remains a probability measure and condition~\eqref{cond:bernstein} holds.
\end{example}

Our goal is to derive bounds on the out-of-sample excess risk
\begin{align*}
    \mathbb{E}[L(X_{t+1}, \hat h_{w})] - \mathbb{E}[L(X_{t+1}, h_{P_{t + 1}}^*)],
\end{align*}
uniformly in $w \in \mathcal{W}$. Here, $h_{P_{t + 1}}^* = \argmin_{h} \mathbb{E}[L(X_{t + 1}, h)]$ is the Bayes-optimal predictor.

Assumption \ref{cond:loss-L2}, stated in Section \ref{sec:assump}, and the inequality $(a + b)^2 \le 2(a^2 + b^2)$ then imply the following for any fixed $\hat{h}_w$:
\begin{align} \label{eq:decomposition}
     \underbrace{\mathbb{E}[L(X_{t+1}, \hat h_{w})] - \mathbb{E}[L(X_{t+1}, h_{P_{t + 1}}^*)]}_{\text{excess out-of-sample risk}}
    &\lesssim \underbrace{\| \hat h_{w} - h_{w}^* \|_{L_2(P_{t+1})}^2}_{\text{learning error}} +  \underbrace{\|h_{w}^* -  h_{P_{t + 1}}^*\|_{L_2(P_{t+1})}^2}_{\text{drift error}}.
\end{align}
Both error terms are measured with respect to the $L_2$ norm under $P_{t+1}$.
We use the term ``learning error'' because it combines the estimation error of $\hat h_w$ with respect to $\bar h_w = \argmin_{h \in \mathcal{H}_w} R^w(h)$ and the approximation error between $\bar h_w$ and $h_w^*$.
The drift error measures the difference between the best hypothesis for the weighted average of past risks and the best hypothesis for the current distribution $P_{t + 1}$. This reflects the nonstationarity of the data and how well the weights $w$ adapt to it, distinguishing our approach from previous works that rely on the discrepancy, which measures the worst-case change in expected loss over the hypothesis class \citep{mansour2009domain}.

In practice, the drift error depends on an interaction between the selected weight $w$ and the nonstationarity of the data-generating process.
The focus of this section is controlling the learning error uniformly over $w \in \mathcal{W}$. Section \ref{sec:driftsquare} gives an analysis of the drift error for concrete drift scenarios.

The bounds depend on several factors. To quantify the dependence structure of the sequence $X_1,X_2, \dots$, we use two mixing coefficients.

\begin{definition}[$\beta$-mixing]\label{def:beta}
For $k\in \mathbb{N}_0$, let $\mathcal{A}_t = \sigma\{X_j : j\le t\}$ and $\mathcal{B}_t = \sigma\{X_j : j\ge t+k\}$. The $\beta$-mixing coefficient at lag $k \ge 0$ is defined as
\begin{align*}
    \beta(k) = \underset{t\geq1}{\sup}\, \beta(\mathcal{A}_t,\mathcal{B}_t)=\underset{t\geq1}{\sup}\, \frac{1}{2} \sup_{\{A_{i}\}_{i} \in \Pi_{t},\{B_{j}\}_{j} \in \Pi'_{t+k}} \sum_{i=1}^{I}\sum_{j=1}^{J} |\mathbb{P}(A_{i}\cap B_{j})-\mathbb{P}(A_{i})\mathbb{P}(B_{j})|,
\end{align*}
where $\Pi_{t}$ and $\Pi'_{t+k}$ denote the sets of all finite partitions of $\mathcal{A}_t$ and $\mathcal{B}_t$, respectively.
\end{definition}

\begin{definition}[$\rho$-mixing]\label{def:rho}
The $\rho$-mixing coefficient at lag $k \ge 0$ is defined as
\begin{align*}
    \rho(k) = \sup_{t \ge 1} \sup_{f,g} \frac{|\mathrm{Cov}(f(X_t),g(X_{t + k}))|}{\sqrt{\mathrm{Var}(f(X_t))\mathrm{Var}(g(X_{t + k}))}},
\end{align*}
with the inner supremum taken over all functions for which the ratio is well-defined,
and its long-run characteristic is defined as $K_\rho =  1 + 2\sum_{k=1}^\infty\rho(k)$.
\end{definition}
The $\beta$-mixing coefficient allows for sharp concentration inequalities via coupling arguments \citep{berbee1979random}, while the $\rho$-mixing coefficient is useful to control local fluctuations of the empirical risk around its optimum \citep[e.g.,][]{DebMukherjee2024}.
These fluctuations further depend on the complexity of the weight and hypothesis classes.
Specifically, denote by $N_1(\epsilon, \mathcal{W})=N(\epsilon,\mathcal{W}, \|\cdot\|_1)$ the covering number of $\mathcal{W}$ with respect to $\|\cdot\|_1$ and by $N_\infty(\epsilon, \mathcal{H}_w)=N(\epsilon, \mathcal{H}_w,\|\cdot\|_\infty)$ the covering number of  $\mathcal{H}_w$ with respect to $\|\cdot\|_\infty$.

When convenient, we write $L_h(\cdot) = L(\cdot, h)$ and $\mathcal{H} = \bigcup_{w \in \mathcal{W}}\mathcal{H}_w$. We frequently use the notation $a \lesssim b$ whenever there is a numerical constant $C$ such that $a \le Cb$. Unless stated otherwise, let $\|w\|$ be the Euclidean norm $\|w\|_2$.

\section{Fast rates for the learning error}\label{sec:main}

\subsection{Assumptions}\label{sec:assump}
We require five key assumptions for Theorem \ref{thm:fast}.
Assumptions  \eqref{cond:loss-L2}--\eqref{cond:loss-upper} provide the geometry needed to relate excess risk and compare norms across time, while \eqref{cond:infty} enables localization and sharper control of dependence effects.
Assumption \eqref{cond:bernstein} is the Bernstein condition that provides the curvature required for fast rates, and \eqref{cond:simplify} imposes standard boundedness and Lipschitz regularity.

\rev{For the weight class $\mathcal W$, assume that every $P_w$ is a probability measure and define $\mathcal P=\{P_1,\ldots,P_{n+1}\}\cup\{P_w:w\in\mathcal W\}$.
}
\begin{itemize}
     \Aconditem\label{cond:loss-L2} There is $C_L < \infty$ such that for any fixed $h\in\mathcal{H}$ and every $P \in \mathcal{P}$, with $h_{P}^* \in \argmin_{h} \mathbb{E}_P[L(X,h)]$,
     \begin{align*}
    \mathbb{E}_P[L(X, h)] - \mathbb{E}_P[L(X, h_{P}^*)] \le C_L \| h - h_{P}^*\|_{L_2(P)}^2.
\end{align*}
\end{itemize}
This assumption is satisfied by many common loss functions, such as squared, Huber, logistic, and zero-one losses, and is only required to hold for deterministic $h$.

Our main result uses a stronger norm $\| f\|_{L_2(\mathcal{P})} = \sup_{P \in \mathcal{P}} \mathbb{E}_{X \sim P}[f(X)^2]^{1/2}$ to control the learning error.
For this norm to be useful, we impose a comparability condition on the set $\mathcal{P}$:
\begin{itemize}
  \Aconditem\label{cond:loss-upper}
  \textit{There is $C_\mathcal{P} < \infty$ such that, for all $h,h'\in \mathcal{H}$,
  \[
     \sup_{P\in\mathcal{P}}\|h-h'\|_{L_2(P)} \le C_\mathcal{P} \inf_{P\in\mathcal{P}}\|h-h'\|_{L_2(P)}.
  \]}
\end{itemize}

\noindent The condition allows us to transfer $L_2(P)$-bounds across time.
It is satisfied, for example, if all $P \in \mathcal{P}$ have densities that are uniformly bounded away from zero and infinity with respect to some dominating measure $Q$---irrespective of $\Hcal$.
Much milder conditions usually suffice if the functions in $\mathcal{H}$ are sufficiently regular (see Example \ref{exp:assumptions}).

The next condition helps to control local deviations in the empirical risk. It is central to our improved rates under polynomial mixing (see Section \ref{sec:explicit-rates}).
\begin{itemize}
  \Aconditem\label{cond:infty}
  \rev{\textit{There is a function $(w,\epsilon)\mapsto C_{w,\infty}(\epsilon)\in[0,\infty)$ such that, for every $w\in\mathcal W$, $\epsilon>0$, and $h,h'\in\mathcal H_w(\epsilon)$,
  \[
    \inf_{P \in \mathcal{P}}\|h - h'\|_{L_2(P)} \ge C_{w,\infty}(\epsilon) \|h - h'\|_{\infty},
  \]}}
\end{itemize}
where $\mathcal H_w(\epsilon)=\bigcup_{v\in\mathcal W:\|v-w\|_1\le\epsilon}\mathcal H_v$ is the neighborhood class of $w$ with radius $\epsilon$.
This formulation allows the hypothesis class to vary locally with the selected weight.
\rev{\noindent Positive $C_{w,\infty}(\epsilon)$ prevents differences between nearby models in the neighborhood from being large in sup-norm while small in $L_2$;
for a fixed hypothesis class, i.e., $\mathcal H_w \equiv\mathcal H$, it is independent of $w$ or $\epsilon$.
\ref{cond:infty} always holds with $C_{w,\infty}(\epsilon)=0$, while $C_{w,\infty}(\epsilon)>0$ can give sharper rates under dependence (see Section \ref{sec:explicit-rates}).
It can often be obtained from inverse inequalities, or by direct calculation as in Section~\ref{sec:applications}.}

\begin{example}\label{exp:assumptions}
\,
\begin{itemize}[leftmargin=1.8em,itemsep=4pt,topsep=4pt,parsep=4pt,partopsep=4pt]
    \item \textit{Scalar class} $\mathcal H=\{h(z) = c:|c| \leq B\}$: \eqref{cond:loss-upper} and~\eqref{cond:infty} hold with $C_{\mathcal P}=1$ \rev{and $C_{w,\infty}(\epsilon)=1$}.

    \item \textit{Linear feature class} $\mathcal H=\{h(z)=\beta^\top\Phi(z):\|\beta\|_2\le B\}$: if there exist constants $0<\underline\lambda\le \bar\lambda<\infty$ such that for all $P\in\mathcal P$,
    \[
    \underline\lambda \le \lambda_{\min}\!\bigl(\mathbb{E}_P[\Phi(Z)\Phi(Z)^\top]\bigr)
    \le
    \lambda_{\max}\bigl(\mathbb{E}_P[\Phi(Z)\Phi(Z)^\top]\bigr)
    \le \bar\lambda,
    \]
    and $\sup_{z\in\mathcal Z}\|\Phi(z)\|_2\le 1$, then \eqref{cond:loss-upper} and \eqref{cond:infty} hold with $C_{\mathcal P}=\sqrt{\bar\lambda/\underline\lambda}$ \rev{and $C_{w,\infty}(\epsilon)=\sqrt{\underline{\lambda}}$}, where $\lambda_{\min},\lambda_{\max}$ are the smallest and largest eigenvalues, respectively.
\end{itemize}
\end{example}

We further impose a margin condition to control the curvature of the excess risk.
\begin{itemize}
  \Aconditem\label{cond:bernstein}
  \textit{Let $P_w=\sum_{t=1}^n w_tP_t$ be a probability measure for any $w \in \Wcal$, and assume that, for any $w \in \mathcal{W}$ and $h \in \mathcal{H}_w$,
\[
R^w(h) - R^w(h_{w}^*) \ge \| h -h_{w}^*\|_{L_2(P_w)}^{2}.
\]}
\end{itemize}

\noindent This is the so-called Bernstein condition \citep{bartlett2006empirical}, typically imposed for fast-rate results, but applied to the mixture distribution $P_w$. It is satisfied, for example, for the square loss (with equality) and a rescaled logistic loss under standard conditions on $\mathcal{H}$.
A final simplifying assumption is made to avoid cluttering our main result and arguments:
\begin{itemize}
    \Aconditem\label{cond:simplify} The loss function $L$ is uniformly 1-Lipschitz and bounded by 1, and all hypotheses $h \in \mathcal{H}$ and $h_w^*$ are uniformly bounded by 1.
\end{itemize}
\noindent Boundedness and Lipschitzness in sup-norm can be relaxed to milder moment conditions through tedious but standard truncation arguments. The constant 1 can be replaced by a straightforward reparametrization of the problem and bounds.

\subsection{Main result}

Our main result is a weight-uniform oracle inequality for the learning error. Fix $\delta\in(0,1/2]$ and let $\mathcal W$ be a class of weight vectors. Define
\begin{align*}
C_1=\sup_{w\in\mathcal W}\|w\|_1,\quad
C_\mathcal W=\inf_{w\in\mathcal W}\|w\|,\quad
B_\mathcal W=\sup_{w\in\mathcal W}\frac{\|w\|_\infty}{\|w\|^2}, \quad
m_\beta=\inf\left\{m:\frac{n}{m}\beta(m)\le\delta\right\}.
\end{align*}
The first three quantities characterize the weight class: $C_1=1$ unless some weights are negative, $C_\mathcal W^{-2}$ is its maximal effective sample size, and $B_\mathcal W$ measures weight spikiness.
The quantity $m_\beta$ is the coupling block length used to control the $\beta$-mixing error.
For each $w\in\mathcal W$, define
\begin{align*}
K_w=4+\log\left(N_1(\epsilon_\mathcal W,\mathcal W)N_\infty(\epsilon_w,\rev{\mathcal H_w(\epsilon_\mathcal W)})^2\right),
\end{align*}
where
\begin{align*}
\epsilon_\mathcal W&=\rev{\min\left\{\frac{C_\mathcal W^3}{64(1+C_\mathcal W+C_1K)},\frac{C_\mathcal W}{64(1+K)}\right\}},\qquad
\epsilon_w=\frac{\|w\|^2}{32C_1}.
\end{align*}
Here, $m_\beta$ and $K_\rho$ account for potential dependence in the data, while $K_w$ captures the complexity of the weight and hypothesis classes. $K$ is a Lipschitz constant for the rate function $r$ defined below.
\begin{theorem}\label{thm:fast}
Fix $K \in (0, \infty)$ and assume \eqref{cond:loss-L2}--\eqref{cond:simplify}.
Let $r\colon [C_\mathcal{W},C_1]\to[0,\infty)$ be increasing and $K$-Lipschitz, and suppose that for all $w\in \mathcal{W}$,
\begin{align}
r(\|w\|)^2 &\ge K_w \|w\|^2\left(C_\mathcal{P}^2K_\rho + m_\beta B_\mathcal{W} \min\left\{2, C_\mathcal{P} \rev{C_{w,\infty}(\epsilon_\mathcal W)^{-1}}r(\|w\|)\right\}\right), \label{cond:rate} \\
r(\|w\|)^2 &\ge 4C_L \inf_{h \in \mathcal{H}_w}\|h-h_w^\ast\|_\infty^2. \label{cond:approx}
\end{align}
Then, with probability at least $1-2\delta$,
\begin{align*}
    \forall w \in \mathcal{W}\colon \quad \|\hat{h}_w - h^*_{w} \|_{L_2(\mathcal{P})}^2 \lesssim  R^w(\hat{h}_w) -R^w(h^*_{w})  \lesssim r(\|w\|)^2 \log^2(1/\delta).
\end{align*}
\end{theorem}

Condition \eqref{cond:rate} reflects the effect of the weight and complexity terms on the estimation error.
It is large when the effective sample size $1/\|w\|^2$ is small or when the complexity of the weight and function classes is high.
We use the convention $C_{w,\infty}(\epsilon_\mathcal W)^{-1}=\infty$ when $C_{w,\infty}(\epsilon_\mathcal W)=0$.
Condition \eqref{cond:approx} is an assumption on the approximation error: the rate cannot be lower than the error stemming from approximating $h_w^*$ by an element $h \in \mathcal{H}_w$.
The function $r$ is required to be $K$-Lipschitz only on the interval $[C_{\mathcal{W}}, C_1]$. Since $C_\mathcal{W} \geq 1/n$, it is not required to be Lipschitz at zero.

\paragraph{Proof idea.} The full proof is deferred to Appendix~\ref{app:proof} and follows a peeling argument with coupling, localization, concentration, and covering. We require two new key techniques:

\emph{Weight-dependent localization.} For every $w$, we peel over the weight-sensitive excess-risk shells
\[
\mathcal H_{k,w}=\left\{h\in\mathcal H_w:R^w(h)-R^w(h_w^*)\le 2^{k+1}r(\|w\|_2)^2\right\}.
\]
Thus $h_w^*$, $\mathcal H_w$, the radius, and the normalization all vary with $w$. The proof jointly discretizes $\mathcal W$ and the corresponding localized classes $\mathcal H_{k,w}$ while preserving one high-probability event for every $w$.

\emph{Sharp dependence control.} Dependence enters in two distinct places. The $\beta$-mixing coefficients control the error of coupling the process to independent blocks, whereas $K_\rho$ controls the variance of localized fluctuations.
\rev{When $C_{w,\infty}(\epsilon_\mathcal W)>0$, the local envelope shrinks with the radius, allowing us to recover an iid-order leading term when $K_\rho< \infty$.}
We discuss this case in the following section.

\begin{remark}
The quadratic Bernstein condition~\eqref{cond:bernstein} can be generalized to $[R^w(h)-R^w(h_w^*)]^\vartheta$, $\vartheta\in[0,1]$, covering additional losses with the usual slower rates as $\vartheta$ decreases.
\end{remark}

\subsection{Explicit rate functions} \label{sec:explicit-rates}

To interpret Theorem~\ref{thm:fast}, suppose that the complexity of $\rev{\mathcal H_w(\epsilon_\mathcal W)}$ grows with the effective sample size $n_{\mathrm{eff}}(w)=\|w\|_2^{-2}$: for some $\alpha\in[0,2)$,
\[
\log N_\infty(\epsilon,\rev{\mathcal H_w(\epsilon_\mathcal W)})\lesssim\|w\|_2^{-\alpha}\log(n/\epsilon),
\qquad \log N_1(\epsilon,\mathcal W)\lesssim\log(n/\epsilon),
\]
and let $C_1,B_{\mathcal W},C_\mathcal P\lesssim1$. The exponent $\alpha$ determines how model complexity is allowed to adapt to the amount of information retained by the weights. If \rev{$C_{w,\infty}(\epsilon_\mathcal W)>0$}, the local $L_2$-to-sup-norm control in~\eqref{cond:infty} yields the admissible \rev{weight-specific} rate
\[
r(\|w\|_2)^2\lesssim K_\rho\|w\|_2^{2-\alpha}\log n
+m_\beta^2B_\mathcal W^2\rev{C_{w,\infty}(\epsilon_\mathcal W)^{-2}}\|w\|_2^{4-2\alpha}\log^2 n.
\]
The first term is the familiar variance contribution at effective sample size $n_{\mathrm{eff}}(w)$; the second is the cost of coupling the dependent sample to independent blocks. Crucially, the latter is of higher order in $\|w\|_2$. Without local sup-norm control (\rev{$C_{w,\infty}(\epsilon_\mathcal W)=0$}), it instead has order $m_\beta B_\mathcal W\|w\|_2^{2-\alpha}\log n$. Appendix~\ref{app:explicit-rates} gives the precise rate functions and conditions.

To sharpen intuition, consider the stationary unweighted case. Let $\mathcal W=\{(1/n,\ldots,1/n)\}$ and $\log N_\infty(\epsilon,\mathcal H)\lesssim\log(n/\epsilon)$, where $\Hcal$ is a fixed parametric class (i.e., $\alpha=0$), $\|w\|_2^2=1/n$, and $B_\mathcal W=1$. Fixing $\delta$ and taking $K_\rho\lesssim1$, we obtain
\[
r(\|w\|_2)^2\lesssim (K_\rho+m_\beta)\frac{\log n}{n}.
\]
Under exponential mixing, where $m_\beta\lesssim\log n$, we recover the optimal rate $n^{-1}\log^2n$ up to logarithms.
Under polynomial mixing $\beta(m)\asymp m^{-(\gamma-1)}$, $\gamma>2$, however, $m_\beta\asymp n^{1/\gamma}$ and the rate becomes $n^{-1+1/\gamma}\log n$, which is slower than optimal, and comparable to existing stationary bounds \citep[e.g.,][]{HangSteinwart2014,DebMukherjee2024}.

With \rev{$C_{w,\infty}(\epsilon_\mathcal W)=C_\infty>0$}, the localized bound above instead yields
\[
r(\|w\|_2)^2\lesssim
\frac{K_\rho\log n}{n}+\frac{m_\beta^2\log^2 n}{n^2\rev{C_\infty^2}}
\lesssim\frac{\log n}{n}+\frac{n^{2/\gamma}\log^2 n}{n^2\rev{C_\infty^2}}
\lesssim\frac{\log n}{n},
\]
in the polynomial case. For fixed \rev{$C_\infty>0$} and $\gamma>2$, the second term is higher order, leaving an iid-order leading rate up to logarithmic factors.
Thus, even in the special case of unweighted ERM with stationary data, Theorem \ref{thm:fast} improves on the state of the art for polynomially mixing data.
In nonparametric settings, where $h_w^* \notin \mathcal{H}$ and $\alpha > 0$, the rates are necessarily slower, and one must carefully balance approximation and estimation errors; see Section \ref{sec:applications} for examples.

\subsection{Lower bound for mean estimation}\label{sec:lower-bounds-main}

To the best of our knowledge, there are no generic lower bounds for nonstationary dependent learning problems, as the optimal rate depends crucially on interactions among the concrete learning problem, the drift model, and the dependence structure.
In particular, it is generally unclear what the worst-case dependence structure is within a given budget of $m_\beta$ or $K_\rho$ for a given problem.
Nevertheless, we can still recover the quantities in our upper bound for concrete models, as we illustrate with the following nonstationary block-dependent process.

For a mean sequence $\mu=(\mu_1,\ldots,\mu_n)$, denote its weighted oracle by $\mu_w^*=\sum_{i=1}^n w_i\mu_i$.
Partition the sample into $K$ consecutive blocks $I_k$ of size $M$ and each $I_k$ into $B$ subblocks $J_{k,b}$ of size $m$. For $i\in J_{k,b}$, let $Y_i=\mu_k+\xi_{k,b}$, with $\xi_{k,b}\sim N(0,I_d)$ \emph{iid}.
Under squared Euclidean loss, $\mu_{w^{(k)}}^*=\mu_k$ and the weighted excess risk is $\|\widehat\mu_{w^{(k)}}-\mu_k\|_2^2$.
Moreover, $\|w^{(k)}\|_2^2=M^{-1}$ and $n_{\mathrm{eff}}(w^{(k)})=M$, whereas repeated observations within a subblock carry the information in one Gaussian vector.
Thus each block provides only $B=M/m$ independent observations from $N(\mu_k,I_d)$.

\begin{proposition}\label{prop:weighted-fano-lower}\label{prop:mdependent-lower}
 In the above model, there exist a nonstationary parameter class $\Theta_{\mathrm{ns}}\subset\mathcal H^K$, $\mathcal H=[-1,1]^d$, and a class $\mathcal W$ of block-uniform weights satisfying $\|w\|_2^2=M^{-1}$ such that, if $d+\log K\le Bd$, then
\[
\inf_{(\widehat\mu_w)}\sup_{\mu\in\Theta_{\mathrm{ns}}}
\mathbb E_\mu\!\left[\sup_{w\in\mathcal W}\|\widehat\mu_w-\mu_w^*\|_2^2\right]
\gtrsim K_\rho\|w\|_2^2\log\!\left[N_1(1/2,\mathcal W)N_\infty(1/2,\mathcal H)\right].
\]
\end{proposition}
Thus the lower bound recovers the effective sample size and complexity terms of~\eqref{cond:rate} up to constants and logarithmic factors. The proof is given in Appendix~\ref{sec:lower-bounds}.

\section{Bounds on the drift error}
\label{sec:driftsquare}

A feature of our risk decomposition \eqref{eq:decomposition} is that the drift error depends purely on the interaction between the marginals $P_t$ and the weight $w$.
In particular, it is unrelated to the hypothesis class $\Hcal_w$, the dependence structure, or the realization of the sample.
This makes the unlearnable part of the problem transparent, but also implies that a full analysis of out-of-sample risk is highly problem-specific.
As an illustrative example, we briefly outline how such an analysis could look for the square loss.

Let $X_t = (Y_t, Z_t) \in [-B, B] \times \mathbb R^p, t \ge 1$, and define $L(X_t, h) = (Y_t - h(Z_t))^2$.
This also covers autoregressive problems in which $Z_t$ contains lagged values of $Y_t$.
The optimal hypothesis under distribution $P_t$ is the conditional expectation $h_{P_t}^*(z) = \mathbb{E}[Y_t | Z_t = z]$. The following lemma is useful for analyzing the drift error under specific structural assumptions.

\begin{lemma}\label{lem:driftlem}
    Assume that the features $Z_i$, $i \ge 1$, have densities $p_{Z_i}$ with respect to a common measure $Q$, $w \ge 0$, and define
        $C_p = \sup_{i \in \{1, \dots, t\}, z \in \mathcal Z} p_{Z_i}(z)/ \sum_{j=1}^t w_j p_{Z_j}(z)$.
    Then
    \[
    \|h_w^*-h_{P_{t+1}}^*\|_{L_2(P_{t+1})}
    \le
    C_p
          \sum_{i=1}^t w_i
        \|h_{P_{i}}^*-h_{P_{t+1}}^*\|_{L_2(P_{t+1})}.
\]
\end{lemma}
The proof involves simple calculations and is given in Appendix \ref{sec:proof-driftlem}.

\begin{example}[Smooth drift]\label{exp:smoothdrift}
     Suppose the oracle path varies smoothly in the sense that
$\|h_{P_{s}}^*-h_{P_{t}}^*\|_{L_2(P_{t})}
\le \kappa | t -s|^\nu$ for some $\nu, \kappa > 0$.
In this setting, putting more weight on recent observations seems appropriate.
For example, the exponential smoothing weights from
Example~\ref{exp:weights} (ii) yield
\[
   \|h_w^*-h_{P_{t+1}}^*\|_{L_2(P_{t+1})} \lesssim C_p \kappa \|w\|^{-2\nu} \lesssim C_p\kappa (1-e^{-\theta})^{-\nu}.%
\]
Thus, a smaller $\theta$ assigns more mass to older observations and increases temporal bias, while smaller $\kappa$ and smoother variation of the oracle hypothesis allow more dispersed weights.
\end{example}

\begin{example}[Structural breaks]\label{exp:breaks}
    Let $\tau_1, \tau_2, \dots \in \mathbb N$ be break points and suppose that  $P_i = P_j$ for all $i, j \in (\tau_k, \tau_{k + 1}], k \ge 1$. This corresponds to a piecewise-stationary process. Consider the uniform weights from Example~\ref{exp:weights} (i) with window size $s$. If \rev{$s\le t-\tau_k$ and $t<\tau_{k + 1}$}, the window intersects only a single segment, and therefore
    $\|h_w^*-h_{P_{t+1}}^*\|_{L_2(P_{t+1})}=0.$
Now suppose the window intersects two segments with end points $\tau_{k}, \tau_{k + 1}$, and \rev{define the current-regime age $d=t-\tau_k$ and the fraction of earlier-regime observations in the window $\pi_s=(1-d/s)_+$.} Then
\[
    \|h_w^*-h_{P_{t+1}}^*\|_{L_2(P_{t+1})}
    \lesssim
    C_p \pi_s
    \|h^*_{P_{\tau_k}}-h^*_{P_{\tau_{k+1}}}\|_{L_2(P_{\tau_{k+1}})},
\]
 and similarly for windows intersecting multiple segments.
A ``good'' weight $w$ must therefore balance the learning error incurred through Theorem \ref{thm:fast} with the magnitude of recent structural breaks.
\end{example}

\section{Excess risk in selected applications}
\label{sec:applications}

We started the paper by separating the excess risk into a learning term and a deterministic drift term in \eqref{eq:decomposition}.
Section \ref{sec:main} gave a detailed analysis of the learning error and its sharpness, even under slow dependence, while Section \ref{sec:driftsquare} analyzed the drift term for concrete drift models in the square loss setting.
We now put these pieces together for selected model classes: we first state the learning rate for a generic weight class $\Wcal$, then recover the familiar iid minimax rate up to logarithmic factors, and finally combine the learning term with the drift analysis under structural breaks to derive concrete excess-risk bounds under optimized weight choices.

We begin with an excess out-of-sample guarantee that isolates the remaining effect of the drift from the learning error.
Theorem \ref{thm:fast} and \eqref{eq:decomposition} imply the following corollary.

\begin{corollary}\label{cor:excess-risk}
    Under the conditions of Theorem \ref{thm:fast}, with probability at least $1 - 2\delta$, for all $w \in \mathcal{W}$,
    \begin{align*}
        \mathbb{E}[L(X_{n + 1}, \hat h_w)] - \mathbb{E}[L(X_{n + 1}, h_{P_{n + 1}}^*)] \lesssim r(\|w\|)^2 \log^2(1/\delta) + \|h_{w}^* -  h_{P_{n + 1}}^*\|_{L_2(P_{n+1})}^2.
    \end{align*}
\end{corollary}
The corresponding cumulative excess-risk guarantee is stated as Theorem~\ref{thm:fastsum} in Appendix~\ref{app:time-uniform}.

\paragraph{Models.} We now specialize Corollary \ref{cor:excess-risk} to three model classes. Throughout, $Y_t=h_t^*(Z_t)+\eta_t$ is bounded with $\sup_{z\in\Zcal}\|z\|_2\le1$, $\mathbb E[\eta_t\mid Z_t]=0$, and we use square loss.
For simplicity, assume $C_1,B_\mathcal W,C_\mathcal P,K_\rho\lesssim1$ and exponential mixing $\beta(m)\lesssim m\exp(-\gamma m)$, $\gamma>0$, and $w\geq 0$.
Table~\ref{tab:rates} states the weight-dependent learning rates and the optimal iid excess risk in columns 2 and 3, respectively, for the model classes discussed below.
We refer to Appendix~\ref{sec:appapplication} for technical details.

\emph{Linear models.}
Let $\mathcal H=\{z\mapsto\beta^\top z:\|\beta\|_2\le B\}$ and $h_t^*(z)=\beta_t^{*\top}z$, where $\|\beta_t^*\|_2\le B$ and $Z_t$ is uniform on $[0,1]^p$, with uniformly well-conditioned covariance matrix.

\emph{Basis expansions.}
Let $Z_t$ be uniform on $[0,1]$, let $h_t^*$ be $1$-Lipschitz, and use step functions on $q(w)$ equal-width bins, with $q(w)$ adapted to the effective sample size.

\emph{ReLU networks.}
Let $Z_t$ be uniform on $[0,1]^p$ and $h_t^*$ be $a$-smooth.
We assume that the ReLU classes have $\nu_w$ neurons, $\ell_w$ layers, parameter bound $b_w\asymp\nu_w$, and architecture size $\nu_w\ell_w\asymp\|w\|_2^{-p/(2a+p)}$ \citep{nagler2026optimal,ou2024covering}.
We abbreviate $\varrho=2a/(2a+p)$ and $\zeta=\rev{\max\{2a/p,2\}}$.

\begin{table}[!ht]
\caption{\rev{Learning and excess-risk rates; $s^*$ is the optimized post-break window.}}
\label{tab:rates}
\centering
\footnotesize
\setlength{\tabcolsep}{5pt}
\begin{tabular}{@{}lccc@{}}
\toprule
Class & Learning error & Excess risk iid & Excess risk at $s^*$ \\
\midrule
Linear & $p\|w\|_2^{2}\log n$ & $p\log n/n$
& $\rev{(p\log n/d)\bigl(1-p\log n/(4C_p^2\Delta^2d)\bigr)}$ \\
Basis & $\|w\|_2^{4/3}\log n$ & $n^{-2/3}\log n$
& $\rev{d^{-2/3}\log n\bigl(1-\log n/(9C_p^2\Delta^2d^{2/3})\bigr)}$ \\
ReLU & $\|w\|_2^{2\varrho}\log^\zeta n$
& $n^{-\varrho}\log^\zeta n$
& $\rev{d^{-\varrho}\log^\zeta n\bigl(1-\varrho^2\log^\zeta n/(4C_p^2\Delta^2d^\varrho)\bigr)}$ \\
\bottomrule
\end{tabular}
\end{table}

\paragraph{Excess risk under structural breaks.}
To see how drift changes the weights, consider a structural break and uniform windows of length $s$.
Then $n_{\mathrm{eff}}(w)=s$, so each leading learning rate has the form $cs^{-b}$. Here $c$ contains the model-specific complexity and logarithmic factors, while $b$ is the exponent governing how quickly the learning error decreases with the effective sample size. For the linear, basis, and ReLU models in Table~\ref{tab:rates}, respectively, $(c,b)=(p\log n,1)$, $(\log n,2/3)$, and $(\log^\zeta n,\varrho)$. If the current regime contains $d$ observations and the oracle changed by $\Delta$, a window extending beyond that regime assigns mass $\pi_s=(1-d/s)_+$ to pre-break data.
By Corollary \ref{cor:excess-risk} and Example~\ref{exp:breaks},
\begin{align*}
\underbrace{\mathbb E[L(X_{t+1},\hat h_{w})-L(X_{t+1},h^*_{P_{t+1}})]}_{\text{excess risk}}
\lesssim
\underbrace{cs^{-b}}_{\text{learning}}
+
\underbrace{C_p^2\Delta^2\pi_s^2}_{\text{drift}},
\end{align*}
where $C_p$ is the density-ratio factor from Lemma~\ref{lem:driftlem} that accounts for covariate shift.
Choosing the window length balances learning and drift: a longer window reduces estimation error, but including observations from before the break introduces bias. Using $d$ observations from the current regime avoids drift and gives error $cd^{-b}$, regardless of the series length.
When the break is small, including older observations can improve this bound because the drift error grows quadratically with their fraction $\pi_s$.
Appendix~\ref{app:structural-break-proof} derives the optimizer $s^*$ and its excess-risk bound.

\paragraph{Simulations.}
Figure~\ref{fig:drift-simulations} illustrates both trade-offs in the scalar location model $Y_t=\mu_t+\eta_t$.
Here $\eta_t$ is a stationary renewal-block process whose $\beta$-coefficients decay only polynomially, but since $C_\infty=1$, Theorem~\ref{thm:fast} still yields an iid-order learning error, as demonstrated in panel (a).
Panels (b)--(c) consider nonstationary scenarios.
A structural break favors a window near the current-regime age, whereas smooth drift favors an intermediate exponential-decay horizon.
Both panels show that the best weights balance lower learning error from longer averaging against greater drift.

\begin{figure}[tp]
\centering
\includegraphics[width=\textwidth]{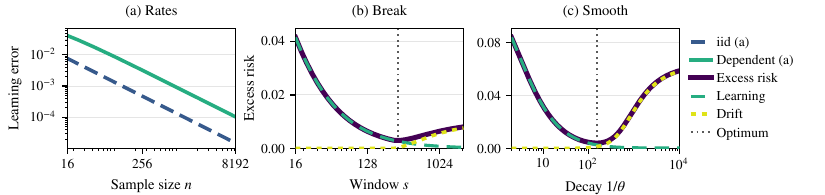}
\caption{Learning error (a) and excess risk under a break (b) and smooth drift (c); see Appendix~\ref{sec:simulations}.}
\label{fig:drift-simulations}
\end{figure}

\section{Discussion}
\label{sec:discussion}
This article separates the excess out-of-sample risk into a learning and a drift error.
We provide fast-rate guarantees for the learning error of weighted empirical risk minimization under nonstationary mixing processes.
The results hold uniformly over arbitrary weight classes and are essentially optimal in unweighted regimes.
We also give an analysis of the drift error for the square loss.
The results are then combined to derive excess-risk bounds under structural breaks for specific models.

A more detailed analysis of the drift error requires nontrivial assumptions.
For example, \citet{kuznetsov2020discrepancy, pmlr-v206-awasthi23b} control nonstationarity by requiring the discrepancy between the recent past and near future to be negligible, so that the drift error can be estimated from the observed data. Such an assumption leads to more concrete risk bounds and algorithms, also within our framework. Our decomposition takes a complementary perspective: it maintains an explicit drift error that captures the unknowable aspects of the underlying data-generating process. An important task for future research is to develop realistic yet general nonstationary models allowing for small drift errors.

Moreover, both error terms depend heavily on the practical choice of $w$.
In applications, weights are often selected by backtesting, drift-detection procedures, or adaptive forecasting rules.
Our theory connects quite naturally to online learning algorithms for updating weights dynamically over time. The uniform bounds provide a starting point for analyzing such data-dependent weighting schemes; the remaining challenge is to relate the selected weights to the prediction-relevant drift term.

\subsection*{AI use statement}

In this work, we used generative AI tools to assist with generating synthetic
data for the simple simulations and with proving mathematical claims, primarily
the rates for the optimal structural-break window $s^*$ and the derivations for
the weight analysis in Appendix~\ref{sec:appweights}. We did not use generative
AI tools for other tasks requiring disclosure, and the remaining categories are
not applicable to this work. Additionally, we used generative AI tools to aid
and polish the writing. We reviewed all AI-assisted text, code, and mathematical
derivations for correctness. We take responsibility for the final content of
this work, including text, claims, and artifacts produced with the aid of
generative AI.

\subsection*{Reproducibility statement}

We include the code needed to reproduce the simulation study in the
supplementary material. The code generates all simulated data and figures
presented in the paper.

\bibliographystyle{iclr2027_conference}
\bibliography{fast_iclr_2027}

\appendix

\section{Preliminary results} \label{app:prelim}

\subsection{Coupling and concentration for \texorpdfstring{$\beta$}{beta}-mixing processes}
    Let $X_1,\dots,X_n$ be a sequence of random variables. We divide this sequence into alternating blocks of size $m\in\mathbb{N}$, assuming w.l.o.g.\ that $n$ is a multiple of $2m$. By maximal coupling \citep[e.g.,][Theorem 5.1]{rio2017asymptotic}, there exist random vectors $ U_{j}^*= (X^*_{(j-1)m+1},\dots,X_{jm}^*) \in \mathcal{X}^m$ such that
\begin{itemize}
    \item $U_j=(X_{(j-1)m+1},\dots,X_{jm}) \overset{d}{=} U_j^*$ for every $j=1,\dots, n/m$,
    \item each of the sequences $(U_{2j}^*)_{j=1,\dots,n/2m}$ and $(U^*_{2j-1})_{j=1,\dots,n/2m}$ is independent,
    \item $ \Pr(\exists j\colon U_j \neq U_j^*) \le \frac{n}{m} \beta(m).$
\end{itemize}
Now define the coupled empirical process $\tilde{\mathbb{G}}^{w,m}$ as $\mathbb{G}^w_n$ with all $X_j$ replaced by $X_j^*$.
The following result is immediate from this definition and the third bullet above.

\begin{lemma}\label{lem:coupleprocess}
    For any class of weights $\mathcal{W}$ and any class of functions $\mathcal{F}$,
    \begin{align*}
    \mathbb{P} \{\exists w \in \mathcal{W}, f  \in \mathcal{F} \colon \mathbb{G}^w_n(f) \neq  \tilde{\mathbb{G}}^{w,m}_n(f)\} \leq  \frac{n}{m} \beta(m).
\end{align*}
\end{lemma}
We further get the following concentration result for the coupled process.

\begin{lemma}\label{lem:bernstein}
Assume $\max_{1 \le i \le n}|f(X_i)| \leq b$ a.s. and $\|f\|_{L_2(\Pcal)}^2 \leq v$. Then, for any $s>0$,
    \begin{align*}
    \mathbb{P}(|\tilde{\mathbb{G}}_n^{w,m}f| > s) \leq
    4\exp\left(-\frac{s^2}{8v \|w\|^2K_\rho +  3mb \|w\|_\infty  s}\right).
\end{align*}
\end{lemma}

\begin{proof}
Define $S_i^*=f(X^*_i)$ and $A_j^w$ by
\begin{align*}
    A_{j}^w=\sum_{i=1}^m w_{(j-1)m+i} (S^*_{(j-1)m+i}- \mathbb{E}[S^*_{(j-1)m+i}]),
\end{align*}
so that
\begin{align*}
    |\tilde{\mathbb{G}}_n^{w,m}f|=\left|\sum_{j=1}^{n/m} A^w_{j}\right| \leq \left|\sum_{j=1}^{n/2m} A_{2j}^w\right| + \left|\sum_{j=1}^{n/2m} A_{2j-1}^w\right|.
\end{align*}
The random variables in each of the sequences $(A_{2j})_{j=1}^{n/2m}$ and $(A_{2j-1})_{j=1}^{n/2m}$ are independent, so we can apply Bernstein's inequality after showing that (a) each block is uniformly bounded and (b) the sum of the variances of the blocks is bounded. For any $j = 1, \dots, n/m$,
\begin{align*}
       \left|A_{j}^w\right| \leq  2mb\|w\|_\infty .
\end{align*}
Further, denoting $\mathcal{I}_j = \{(j-1)m+1, \dots, jm \}$ and using
 \begin{align*}
    \max_{i}\var[S_i^*] &\leq  \|f\|_{L_2(\Pcal)}^2 \leq v,
 \end{align*}
 we get
\begin{align*}
    \var[A_{j}^w]  = \mathbb{E}[(A_{j}^w)^2] &=\sum_{i,k \in \mathcal{I}_j}w_i w_k \operatorname{Cov}(S_i^*-\mathbb{E}S_i^*,S_k^*-\mathbb{E}S_k^*)  \\
    &\leq  \sum_{i,k \in \mathcal{I}_j}|w_i w_k| \rho(|k-i|) \sqrt{\var[S_i^*] \var[S_k^*]}\\
    &\leq v \left(\rho(0) \sum_{i \in \mathcal{I}_j}w_i^2+2\sum_{u=1}^{m-1}\rho(u)\sum_{i \in \mathcal{I}_j}|w_iw_{i+u}|\right) \\
    &\leq v\left( \sum_{i \in \mathcal{I}_j}w_i^2+2\sum_{u=1}^{m-1}\rho(u)\sum_{i \in \mathcal{I}_j}w_i^2\right) \\
    &\leq v K_\rho \sum_{i \in \mathcal{I}_j}w_i^2,
\end{align*}
where we used the Cauchy--Schwarz inequality in the penultimate step. Summing over the even blocks gives
\begin{align*}
    \sum_{j=1}^{n/2m}  \mathbb{E}[(A_{2j}^w)^2] &\leq vK_\rho \sum_{j=1}^{n/2m}  \sum_{i \in \mathcal{I}_{2j}}w_i^2
    \leq v\|w\|^2K_\rho,
\end{align*}
so Bernstein's inequality for independent random variables yields
\begin{align*}
       \mathbb{P}\left(\left|\sum_{j=1}^{n/2m} A_{2j}^w\right| > s\right) &\leq  2\exp\left(-\frac{s^2}{2v\|w\|^2K_\rho + \frac{4}{3}mb\|w\|_\infty s} \right).
\end{align*}
Since the same variance bound and argument hold for the odd blocks, we finally get
\begin{align*}
    \mathbb{P}(|\tilde{\mathbb{G}}_n^{w,m}f| > s)
    &\leq \mathbb{P}\left(\left|\sum_{j=1}^{n/2m} A_{2j}^w\right| > s/2\right) +  \mathbb{P}\left(\left|\sum_{j=1}^{n/2m} A_{2j-1}^w\right| > s/2\right) \\
    &\leq4\exp\left(-\frac{s^2}{8v\|w\|^2K_\rho + 3mb\|w\|_\infty s} \right),
\end{align*}
as claimed.
\end{proof}

\subsection{Weight and hypothesis discretization}

\begin{lemma}\label{lem:radiuslipschitz}
For any $w, w' \in \mathcal{W}$ satisfying the conditions of Theorem \ref{thm:fast},
\begin{align*}
    |r(\|w\|)^{-2} - r(\|w'\|)^{-2} | \leq \frac{2K \|w - w'\|_1}{C_\mathcal{W}^{3 }}.
\end{align*}
\end{lemma}
\begin{proof}
    First note that \eqref{cond:rate} implies
         $$r(\|w\|)^2  \ge \|w\|^2 C_\Pcal^2 K_\rho \ge \|w\|^2 \ge C_\mathcal{W}^2. $$
    Set  $u_1 = \|w\|$ and $u_2 = \|w'\|$ and observe that
\begin{align*}
    |r(u_1)^{-2} - r(u_2)^{-2} | = \frac{|r(u_1)  - r(u_2)|(r(u_1) + r(u_2))}{r(u_1)^2 r(u_2)^2} \leq \frac{2 |r(u_1) - r(u_2)|}{C_\mathcal{W}^{3}}
    &\leq \frac{2K|u_1 - u_2|}{ C_\mathcal{W}^{3}} .
\end{align*}
Now the claim follows from $|\|w\| - \|w'\|| \leq \|w - w'\| \le \|w - w'\|_1$.
\end{proof}

\begin{lemma}\label{lem:epsilonnet}
For any $(h,h'),(\tilde{h},\tilde{h}') \in \mathcal{H} \times \mathcal{H}$ and $w,\tilde{w} \in \mathcal{W}$ satisfying the conditions of Theorem \ref{thm:fast},
    \begin{align*}
         &\quad \, |r(\|w\|)^{-2}\mathbb{G}_n^{w}(L_{h}-L_{h'}) -r(\|\tilde{w}\|)^{-2} \mathbb{G}_n^{\tilde{w}}(L_{\tilde{h}}-L_{\tilde{h}'})| \\
         &\leq 2C_1\|\tilde{w}\|^{-2}(\|h- \tilde{h}\|_\infty + \|h'- \tilde{h}'\|_\infty) +\rev{8C_\mathcal{W}^{-3}(1+C_\mathcal W+C_1K)}\|w-\tilde w\|_1.
    \end{align*}
\end{lemma}

\begin{proof}
By the triangle inequality
\begin{align*}
     &\quad \; |r(\|w\|)^{-2}\mathbb{G}_n^{w}(L_{h}-L_{h'}) -r(\|\tilde{w}\|)^{-2} \mathbb{G}_n^{\tilde{w}}(L_{\tilde{h}}-L_{\tilde{h}'})| \\
     &=| r(\|w\|)^{-2}[\mathbb{G}_n^{w}(L_{h}-L_{h'})-\mathbb{G}_n^{\tilde{w}}(L_{h}-L_{h'})] + [r(\|w\|)^{-2}-r(\|\tilde{w}\|)^{-2}]\mathbb{G}_n^{\tilde{w}}(L_{h}-L_{h'})| \\
   &\le r(\|w\|)^{-2}|\mathbb{G}_n^{w}(L_{h}-L_{h'})-\mathbb{G}_n^{\tilde{w}}(L_{h}-L_{h'})|
    +|r(\|w\|)^{-2}-r(\|\tilde{w}\|)^{-2}|\, |\mathbb{G}_n^{\tilde{w}}(L_{h}-L_{h'})|.
\end{align*}
We begin by analyzing the first term. Observe that
\begin{align*}
   |\mathbb{G}_n^{w}(L_{h}) - \mathbb{G}_n^{\tilde{w}}(L_{h})|
   &= \left|\sum_{t=1}^n (w_t-\tilde{w}_t) (\mathbb{E}[L(X_t,h)]-L(X_t,h))\right| \\
    &\leq 2\sum_{t=1}^n |w_t-\tilde{w}_t| \\
    &= 2 \|w-\tilde{w}\|_1.
\end{align*}
It further holds $r(\|w\|)\geq \|w\|\geq C_\mathcal{W}$. Thus
\begin{align*}
   &\quad \, r(\|w\|)^{-2}|\mathbb{G}_n^{w}(L_{h}-L_{h'})-\mathbb{G}_n^{\tilde{w}}(L_{h}-L_{h'})|   \\
   &\leq C_\mathcal{W}^{-2} |(R^w-R_n^w)(h) - (R^w-R^w_n)(h')-(R^{\tilde{w}}-R_n^{\tilde{w}})(h) + (R^{\tilde{w}}-R^{\tilde{w}}_n)(h')| \\
   &\leq C_\mathcal{W}^{-2}\Bigl(|(R^w-R_n^w)(h)-(R^{\tilde{w}}-R_n^{\tilde{w}})(h)| +|(R^w-R^w_n)(h')-(R^{\tilde{w}}-R^{\tilde{w}}_n)(h')|\Bigr) \\
   &\leq 4C_\mathcal{W}^{-2}\|w-\tilde{w}\|_1 .
\end{align*}
We now bound the second term by Lemma \ref{lem:radiuslipschitz}:
\begin{align*}
    &\quad \,|r(\|w\|)^{-2}-r(\|\tilde{w}\|)^{-2}| |\mathbb{G}_n^{\tilde{w}}(L_{h}-L_{h'})| \\
    &\leq 2KC_\mathcal{W}^{-3}\|w-\tilde{w}\|_1 |\mathbb{G}_n^{\tilde{w}}(L_{h}-L_{h'})| \\
    &\leq 2KC_\mathcal{W}^{-3}\|w-\tilde{w}\|_1 \sum_{t=1}^n |\tilde{w}_t| |L(X_t,h)-L(X_t,h')+\mathbb{E}[L(X_t,h)-L(X_t,h')]| \\
    &\leq 8KC_\mathcal{W}^{-3} C_1\|w-\tilde{w}\|_1 .
\end{align*}
It remains to replace $\mathbb{G}_n^{\tilde{w}}(L_{h}-L_{h'})$ by $\mathbb{G}_n^{\tilde{w}}(L_{\tilde{h}}-L_{\tilde{h}'})$. We have
\begin{align*}
    |\mathbb{G}_n^{\tilde{w}}(L_{h}-L_{h'}) - \mathbb{G}_n^{\tilde{w}}(L_{\tilde{h}}-L_{\tilde{h}'})| \leq |\mathbb{G}_n^{\tilde{w}}(L_{h}-L_{\tilde{h}})| +|\mathbb{G}_n^{\tilde{w}}(L_{\tilde{h}'}-L_{h'})|,
\end{align*}
and note that
\begin{align*}
    |\mathbb{G}_n^{\tilde{w}}(L_{h}-L_{\tilde h})| &\le |(R^{\tilde{w}}-R_n^{\tilde{w}})(h) - (R^{\tilde{w}}-R^{\tilde{w}}_n)(\tilde{h})| \\
    &\leq \sum_{t=1}^n |\tilde{w}_t| |L(X_t,h)-L(X_t,\tilde{h})+\mathbb{E}[L(X_t,h)-L(X_t,\tilde{h})]| \\
    &\leq \sum_{t=1}^n |\tilde{w}_t| |L(X_t,h)-L(X_t,\tilde{h})|+\mathbb{E}|L(X_t,h)-L(X_t,\tilde{h})| \\
    &\leq \sum_{t=1}^n |\tilde{w}_t| (\|h-\tilde{h}\|_\infty + \|h-\tilde{h}\|_\infty) \\
    &\leq2C_1\|h-\tilde{h}\|_\infty.
\end{align*}
Because the same argument holds for $|\mathbb{G}_n^{\tilde{w}}(L_{\tilde{h}'}-L_{h'})|$, we have
\begin{align*}
    |\mathbb{G}_n^{\tilde{w}}(L_{h}-L_{\tilde{h}})| +|\mathbb{G}_n^{\tilde{w}}(L_{\tilde{h}'}-L_{h'})| \leq 2C_1 (\|h-\tilde{h}\|_\infty +\|h'-\tilde{h}'\|_\infty ),
\end{align*}
which, using $r(\|\tilde{w}\|)^{-2} \le \|\tilde{w}\|^{-2}$, yields
\begin{align*}
    r(\|\tilde{w}\|)^{-2}|\mathbb{G}_n^{\tilde{w}}(L_{h}-L_{h'}) - \mathbb{G}_n^{\tilde{w}}(L_{\tilde{h}}-L_{\tilde{h}'})| \le 2C_1\|\tilde{w}\|^{-2}(\|h-\tilde{h}\|_\infty +\|h'-\tilde{h}'\|_\infty ).
\end{align*}
The claim follows after summing the bounds and noting that
\begin{align*}
    \rev{4C_\mathcal{W}^{-2}
    +8KC_\mathcal{W}^{-3} C_1}
    \leq \rev{8C_\mathcal{W}^{-3}(1+C_\mathcal W+C_1K)}
\end{align*}
\rev{and multiplying by $\|w-\tilde w\|_1$ gives the claim.}
\end{proof}

\section{Proofs of main results}\label{app:proof}

\subsection{Proof of Theorem \texorpdfstring{\ref{thm:fast}}{\ref*{thm:fast}}}
For any $w \in \mathcal{W}$, define $\bar h_w \in \mathcal{H}_w$ as the closest approximator of $h_w^*$. By \eqref{cond:approx}, $\|\bar h_w - h_w^*\|_\infty^2 \le r(\|w\|)^2/(4C_L)$.
   Invoking \eqref{cond:loss-L2}, we get
   \begin{align*}
      R^w(\hat{h}_w) -R^w(h^*_{w}) &\le R^w(\hat{h}_w) -R^w(\bar h_{w}) +  R^w(\bar{h}_w) -R^w(h_{w}^*)  \\
       &\le R^w(\hat{h}_w) -R^w(\bar h_{w}) + C_L\|\bar h_w - h_w^*\|_{L_2(\Pcal)}^2 \\
       &\le R^w(\hat{h}_w) -R^w(\bar h_{w}) + r(\|w\|)^2 / 4.
   \end{align*}
    Since $R_n^w(\hat{h}_w) \leq R_n^w(\bar h_{w})$, we have
    \begin{align*}
       R^w(\hat{h}_w) -R^w(\bar h_{w}) &=  (R^w-R_n^w)(\hat{h}_w) - (R^w-R^w_n)(\bar h_w) +  R_n^w(\hat{h}_w) - R_n^w(\bar h_{w})\\
        &\leq   (R^w-R^w_n)(\hat{h}_w) - (R^w-R_n^w)(\bar h_{w}) \\
        &= \mathbb{G}^w_n(L_{\hat{h}_w}-L_{\bar h_w}),
    \end{align*}
where $L_h = L(\cdot, h)$.
We now want to bound $\mathbb{P} \{\sup_w r(\|w\|)^{-2}(R^w(\hat{h}_w) -R^w(h_{w}^*)) > 2^{M} \}$.
The previous two displays imply
   \begin{align*}
        r(\|w\|)^{-2}(R^w(\hat{h}_w) -R^w( h_{w}^*)) \le r(\|w\|)^{-2}\mathbb{G}^w_n(L_{\hat{h}_w}-L_{\bar h_w}) + 1/4.
   \end{align*}
For every $k \ge 1$ and $w \in \mathcal{W}$, define the sets $H_{k, w} = \{h \in \mathcal{H}_w\colon R^w(h) - R^w(h_w^*) \le 2^{k +1} r(\|w\|)^{2} \}$. Then
\begin{align*}
     & \quad \, \mathbb{P}\left(\sup_{w \in \mathcal{W}} r(\|w\|)^{-2}  (R^w(\hat{h}_w) -R^w(h^*_{w})) > 2^M\right)                                                                                                                 \\
     & =\mathbb{P}\left( \bigcup_{k = M}^\infty  \left\{2^{k} < \sup_{w \in \mathcal{W}} r(\|w\|)^{-2}  (R^w(\hat{h}_w) -R^w(h^*_{w})) \le 2^{k +1}\right\}\right)                                                                             \\
     & =\mathbb{P}\left( \bigcup_{k = M}^\infty  \left\{ \sup_{w \in \mathcal{W}} r(\|w\|)^{-2}  (R^w(\hat{h}_w) -R^w(h^*_{w}))\mathbbm{1}_{\hat h_w \in H_{k, w}}   >  2^k - 1/4\right\} \right)                                              \\
     & \le\mathbb{P}\left( \bigcup_{k = M}^\infty  \left\{ \sup_{w \in \mathcal{W}} r(\|w\|)^{-2} \mathbb{G}^w_n(L_{\hat{h}_w}-L_{\bar h_w}) \mathbbm{1}_{\hat h_w \in H_{k, w}} > 2^k - 1/4\right\}\right)                                    \\
     & \le\mathbb{P}\left( \bigcup_{k = M}^\infty  \left\{ \sup_{w \in \mathcal{W}} \sup_{h \in H_{k, w}} r(\|w\|)^{-2} |\mathbb{G}^w_n(L_{h}-L_{\bar h_w})|>  2^k - 1/4  \right\} \right)                                                     \\
     & \le\mathbb{P}\left( \bigcup_{k = M}^\infty  \left\{ \sup_{w \in \mathcal{W}} \sup_{h \in H_{k, w}} r(\|w\|)^{-2} |\tilde{\mathbb{G}}_n^{w,m_\beta}(L_{h}-L_{\bar h_w})|>  2^k - 1/4 \right\} \right) + \frac{n}{m_\beta} \beta(m_\beta) \\
     & \le \sum_{k = M}^\infty \mathbb{P}\left(  \sup_{w \in \mathcal{W}} \sup_{h \in H_{k, w}} r(\|w\|)^{-2}| \tilde{\mathbb{G}}_n^{w,m_\beta}(L_{h}-L_{\bar h_w})|>  2^k - 1/4   \right) + \frac{n}{m_\beta} \beta(m_\beta),
\end{align*}
    where the last two inequalities follow from Lemma \ref{lem:coupleprocess} and the union bound, respectively. By definition of $m_\beta$, $(n/m_\beta) \beta(m_\beta) \le \delta$.
    \rev{Further, for any $h \in H_{k, w}$, since $h,\bar h_w\in\mathcal H_w$, \eqref{cond:loss-upper}, the triangle inequality under $P_w$, \eqref{cond:bernstein}, and \eqref{cond:approx} imply, for $k\ge M$,}
\begingroup
  \begin{align*}
        \|h - \bar h_w\|_{L_2(\Pcal)}
        &\le C_\Pcal\|h-\bar h_w\|_{L_2(P_w)} \\
        &\le C_\Pcal\bigl(\|h-h_w^*\|_{L_2(P_w)}+\|\bar h_w-h_w^*\|_{L_2(P_w)}\bigr)\\
        &\le C_\Pcal\bigl([R^w(h)-R^w(h_w^*)]^{1/2}+r(\|w\|)\bigr)\\
        &\le C_\Pcal(2^{(k+1)/2}+1)r(\|w\|)
        \le C_\Pcal2^{k/2+1}r(\|w\|).
  \end{align*}
\endgroup
  Thus, defining $H_{k, w}' = \{(h, h') \in \mathcal{H}_w \times \mathcal{H}_w \colon \|h - h'\|_{L_2(\Pcal)} \le C_\Pcal 2^{k/2 + 1} r(\|w\|)\} $ and substituting above yields
 \begin{align*}
       &\quad \, \mathbb{P}\left(\sup_{w \in \mathcal{W}} r(\|w\|)^{-2}  (R^w(\hat{h}_w) -R^w(h^*_{w})) > 2^M\right)   \\
       &\le \sum_{k = M}^\infty \mathbb{P}\left(  \sup_{w \in \mathcal{W}} \sup_{(h, h')  \in H_{k, w}' } r(\|w\|)^{-2} |\tilde{\mathbb{G}}_n^{w,m_\beta}(L_{h}-L_{h'})|>  2^k - 1/4  \right) + \delta.
 \end{align*}
\rev{Let
\begin{align} \label{eq:local-union-proofs}
    \overline H'_{k,w}=\left\{(h,h')\in(\mathcal H_w(\epsilon_\mathcal W))^2:
\|h-h'\|_{L_2(\Pcal)}\le C_\Pcal2^{(k+3)/2}r(\|w\|)\right\}.
\end{align}
}

A covering argument (Lemma \ref{lem:covering} below) shows that for every $k \ge \rev{M}$,
\begin{align*}
&\quad \, \mathbb{P} \left(\sup_{w \in \mathcal{W},(h,h') \in H'_{k,w}} r(\|w\|)^{-2}|\tilde{\mathbb{G}}_n^{w,m_\beta}(L_{h}-L_{h'})| > 2^k-\frac{1}{4}\right) \\
    &\leq N_1(\epsilon_\mathcal{W}, \mathcal{W}) \rev{\sup_{w \in \mathcal{W}}\Bigg\{} N_\infty^2(\rev{\epsilon_w,\mathcal H_w(\epsilon_\mathcal W)})
    \rev{\sup_{(h,h')\in\overline H'_{k,w}}}
    \mathbb P\left(r(\|w\|)^{-2}|\tilde{\mathbb G}_n^{w,m_\beta}(L_h-L_{h'})|>2^{k-1}\right)\rev{\Bigg\}}\\
    &\leq \sup_{w\in\mathcal W}e^{K_w}
    \rev{\sup_{(h,h')\in\overline H'_{k,w}}}
    \mathbb P\left(r(\|w\|)^{-2}|\tilde{\mathbb G}_n^{w,m_\beta}(L_h-L_{h'})|>2^{k-1}\right).
\end{align*}
Substituting in the previous display, we get
\begin{align*}
    &\quad \, \mathbb{P} \left(\sup_w r(\|w\|)^{-2}(R^w(\hat{h}_w) -R^w(h^*_{w})) > 2^{M} \right) \\
    &\leq \sum_{k=M}^\infty\sup_{w\in\mathcal W}e^{K_w}
    \rev{\sup_{(h,h')\in\overline H'_{k,w}}}
    \mathbb P\left(r(\|w\|)^{-2}|\tilde{\mathbb G}_n^{w,m_\beta}(L_h-L_{h'})|>2^{k-1}\right)+\delta\\
    &\leq  2\delta,
\end{align*}
where the last inequality follows from Lemma \ref{thm:bernsteinfast} upon choosing
$M = \rev{\lceil 2\log_2(96 +  24 \log(34 / \delta)) \rceil}$.
This concludes the proof of the main theorem. We now state and prove the intermediate lemmas.

\begin{lemma}\label{lem:covering}
Suppose the conditions of Theorem \ref{thm:fast} hold and let $H_{k, w}' = \{(h, h') \in \mathcal{H}_w \times \mathcal{H}_w \colon \|h - h'\|_{L_2(\Pcal)} \le C_\Pcal 2^{k/2 + 1} r(\|w\|)\}$. \rev{Let $\overline H'_{k,w}$ be the union class defined in \eqref{eq:local-union-proofs}.} Then, for any $k\geq 1$,
    \begin{align*}
       &\quad \, \mathbb{P} \left( \sup_{w \in \mathcal{W},(h,h') \in H'_{k,w}} r(\|w\|)^{-2}|\mathbb{G}_n^{w}(L_{h}-L_{h'})| > 2^k-\frac{1}{4}\right) \\
    &\leq N_1(\epsilon_\mathcal{W},\mathcal W)\rev{\sup_{w\in\mathcal W}\Bigg\{}N_\infty^2(\rev{\epsilon_w,\mathcal H_w(\epsilon_\mathcal W)})
    \rev{\sup_{(h,h')\in\overline H'_{k,w}}}
    \mathbb P\left(r(\|w\|)^{-2}|\mathbb G_n^w(L_h-L_{h'})|>2^{k-1}\right)\rev{\Bigg\}}.
    \end{align*}
\end{lemma}

\begin{proof}
By Lemma \ref{lem:epsilonnet}, for $h=\tilde{h}$ and $h'=\tilde{h}'$:
\begin{align*}
         r(\|w\|)^{-2}|\mathbb{G}_n^{w}(L_{h}-L_{h'})| \leq r(\|\tilde{w}\|)^{-2}|\mathbb{G}_n^{\tilde {w}}(L_{h}-L_{h'})|+ \rev{8C_\mathcal{W}^{-3}(1+C_\mathcal W+K C_1)}\|w-\tilde w\|_1.
\end{align*}
Let $\mathcal M(\epsilon_\mathcal{W})$ be an $\epsilon_\mathcal{W}$-net of $\mathcal{W}$ with respect to the $\| \cdot \|_1$-norm such that $|\mathcal M(\epsilon_\mathcal{W})|\leq N_1(\epsilon_\mathcal{W}, \mathcal{W})$. In particular, for any $w \in \mathcal{W}$, there is $\tilde{w} \in \mathcal{M}(\epsilon_\mathcal{W})$ such that $\|w-\tilde{w}\|_1 \leq \epsilon_\mathcal{W}$. \rev{Since $\epsilon_\mathcal W\le C_\mathcal W^3/[64(1+C_\mathcal W+C_1K)]$, the weight-discretization error is at most $1/8$.}
\rev{Moreover, $\mathcal H_w\subseteq\mathcal H_{\tilde w}(\epsilon_\mathcal W)$. Lipschitzness of $r$ and the definition of $\epsilon_\mathcal W$ give
\[
r(\|w\|)\le r(\|\tilde w\|)+K\epsilon_\mathcal W
\le \frac{65}{64}r(\|\tilde w\|)\le \sqrt{2}\,r(\|\tilde w\|),
\]
because $K\epsilon_\mathcal W\le C_\mathcal W K/[64(1+K)]\le r(\|\tilde w\|)/64$. Hence every $H'_{k,w}$ assigned to $\tilde w$ is contained in
\[
\left\{(h,h')\in(\mathcal H_{\tilde w}(\epsilon_\mathcal W))^2:
\|h-h'\|_{L_2(\Pcal)}\le C_\Pcal2^{(k+3)/2}r(\|\tilde w\|)\right\}.
\]}
    \begin{align*}
       &\quad \, \mathbb{P} \left( \sup_{w \in \mathcal{W},(h,h') \in H'_{k,w}} r(\|w\|)^{-2}|\mathbb{G}_n^{w}(L_{h}-L_{h'})| > 2^k-\frac{1}{4}\right) \\
    &\leq \mathbb{P}\left(\sup_{\substack{\tilde w\in\mathcal M(\epsilon_\mathcal W),\\
    (h,h')\in\rev{\bigcup_{v\in\mathcal W:\,\|v-\tilde w\|_1\le\epsilon_\mathcal W}H'_{k,v}}}}
    r(\|\tilde w\|)^{-2}|\mathbb G_n^{\tilde w}(L_h-L_{h'})|>2^k-\frac38\right).
    \end{align*}
   Setting $\tilde{w}=w$ in Lemma \ref{lem:epsilonnet} gives
    \begin{align*}
          r(\|\tilde{w}\|)^{-2}|\mathbb{G}_n^{\tilde{w}}(L_{h}-L_{h'})|
         &\leq r(\|\tilde{w}\|)^{-2} |\mathbb{G}_n^{\tilde{w}}(L_{\tilde{h}}-L_{\tilde{h}'})|+ 2C_1 \|\tilde{w}\|^{-2} (\|h- \tilde{h}\|_\infty + \|h'- \tilde{h}'\|_\infty).
    \end{align*}
\rev{For every $\tilde w \in \mathcal M(\epsilon_\mathcal{W})$, cover each component of $\mathcal H_{\tilde w}(\epsilon_\mathcal W)$ at radius $\epsilon_{\tilde w}=\|\tilde w\|^2/(32C_1)$ in sup-norm.
Taking the Cartesian product of two such covers gives a cover of the product class for the metric $d_\oplus((h,h'),(\tilde h,\tilde h'))=\|h-\tilde h\|_\infty+\|h'-\tilde h'\|_\infty$ at radius
$2\epsilon_{\tilde w}=\|\tilde w\|^2/(16C_1)$ with at most $N_\infty(\epsilon_{\tilde w},\mathcal H_{\tilde w}(\epsilon_\mathcal W))^2$ cells. Retain only cells intersecting $\bigcup_{v\in\mathcal W:\,\|v-\tilde w\|_1\le\epsilon_\mathcal W}H'_{k,v}$ and choose one representative from that intersection. These representatives belong to the localized union and hence, by the preceding argument, to $\overline H'_{k,\tilde w}$. Any point in a retained cell is within distance $2 \cdot 2\epsilon_{\tilde w}$ of its representative.}
\rev{Consequently, the discretization error is at most
$2C_1\|\tilde w\|^{-2}(4\epsilon_{\tilde w})=1/4$. Denoting the resulting set of representatives by $\mathcal N'_{k,\tilde w}$, the last display and a union bound give}
\begingroup
\begin{align*}
    &\quad\mathbb P\left(\sup_{\substack{\tilde w\in\mathcal M(\epsilon_\mathcal W),\\
    (h,h')\in\bigcup_{v:\,\|v-\tilde w\|_1\le\epsilon_\mathcal W}H'_{k,v}}}
    r(\|\tilde w\|)^{-2}|\mathbb G_n^{\tilde w}(L_h-L_{h'})|>2^k-\frac38\right)\\
    &\le \mathbb P\left(\sup_{\tilde w\in\mathcal M(\epsilon_\mathcal W),\,
    (\tilde h,\tilde h')\in\mathcal N'_{k,\tilde w}}
    r(\|\tilde w\|)^{-2}|\mathbb G_n^{\tilde w}(L_{\tilde h}-L_{\tilde h'})|>\rev{2^k-\frac58}\right)\\
    &\le \sum_{\tilde w\in\mathcal M(\epsilon_\mathcal W)}
    \sum_{(\tilde h,\tilde h')\in\mathcal N'_{k,\tilde w}}
    \mathbb P\left(r(\|\tilde w\|)^{-2}|\mathbb G_n^{\tilde w}(L_{\tilde h}-L_{\tilde h'})|>2^{k-1}\right)\\
    &\le N_1(\epsilon_\mathcal W,\mathcal W)\rev{\sup_{w\in\mathcal W}\Bigg\{}N_\infty^2(\epsilon_w,\mathcal H_w(\epsilon_\mathcal W))
    \sup_{(h,h')\in\overline H'_{k,w}}
    \mathbb P\left(r(\|w\|)^{-2}|\mathbb G_n^w(L_h-L_{h'})|>2^{k-1}\right)\rev{\Bigg\}},
\end{align*}
\endgroup
which completes the proof.
\end{proof}

\begin{lemma}\label{thm:bernsteinfast}
    Let the conditions of Theorem \ref{thm:fast}  hold and $\delta \in (0,1)$ be arbitrary.
    Then, for $M = \rev{\lceil 2\log_2(96 + 24 \log(34 / \delta)) \rceil}$,
    \begin{align*}
       \rev{\sum_{k=M}^\infty\sup_{w\in\mathcal W}e^{K_w}}
       \sup_{\substack{h, h' \in \rev{\mathcal H_w(\epsilon_\mathcal W)}:\\ \|h - h'\|_{L_2(\Pcal)}^2 \leq C_\Pcal^2 \rev{2^{k + 3}} r(\|w\|)^2}}
       \mathbb{P}\!\left(  |\tilde{\mathbb{G}}_n^{w,m_\beta}(L_{h}-L_{h'})| > 2^{k - 1}  r(\|w\|)^{2}\right)
       \le  \delta.
    \end{align*}
\end{lemma}

\begin{proof}
Let $w\in \mathcal{W}$ be arbitrary and \rev{$h,h'\in\mathcal H_w(\epsilon_\mathcal W)$} be such that
$\|h-h'\|_{L_2(\Pcal)}^2 \le C_\Pcal^2 \rev{2^{k + 3}} r^2$, where here and below we abbreviate
$r = r(\|w\|)$.
Then, by assumption,
\begin{align*}
    |L_h(X_i) - L_{h'}(X_i)|
    &\le |h(X_i) - h'(X_i)|
    \le \|h-h'\|_\infty\\
    &\le \rev{C_{w,\infty}(\epsilon_\mathcal W)^{-1}} \|h-h'\|_{L_2(\Pcal)}
    \le C_\Pcal \rev{C_{w,\infty}(\epsilon_\mathcal W)^{-1}} \rev{2^{(k + 3)/2}} r,
\end{align*}
and trivially $ |L_h(X_i) - L_{h'}(X_i)| \le 2$.
Similarly,
\begin{align*}
    \var[L_h(X_i) - L_{h'}(X_i)]
    \le \|L_h - L_{h'}\|_{L_2(\Pcal)}^2
    \le \|h - h'\|_{L_2(\Pcal)}^2
    \le C_\Pcal^2 \rev{2^{k + 3}}r^2.
\end{align*}
We apply Lemma~\ref{lem:bernstein} with
\[
v = C_\Pcal^2 \rev{2^{k + 3}} r^2,\qquad
b = \min\{2,\, C_\Pcal \rev{C_{w,\infty}(\epsilon_\mathcal W)^{-1}} \rev{2^{(k + 3)/2}} r\}, \qquad
s = 2^{k - 1} r^2,
\]
to obtain
\begin{align}
    \begin{aligned} \label{eq:bernstein-intermediate}
        &\quad \, \mathbb{P}\!\left(  |\tilde{\mathbb{G}}_n^{w,m_\beta}(L_{h}-L_{h'})| > 2^{k-1} r^2   \right) \\
    &\le
    4\exp\!\left(-\frac{ 2^{2(k - 1)} r^2}{
    8 C_\Pcal^2 \rev{2^{k + 3}} \|w\|^2K_\rho
    +  3 m_\beta \|w\|_\infty \min\{2, C_\Pcal \rev{C_{w,\infty}(\epsilon_\mathcal W)^{-1}} \rev{2^{(k + 3)/2}} r\} \,2^{k - 1} }\right).
    \end{aligned}
\end{align}

\noindent \emph{Case: \rev{$C_{w,\infty}(\epsilon_\mathcal W)=0$}.}
Then $b=2$, and \eqref{eq:bernstein-intermediate} and \eqref{cond:rate} imply
\[
\mathbb{P}\!\left(  |\tilde{\mathbb{G}}_n^{w,m_\beta}(L_{h}-L_{h'})| > 2^{k-1} r^2   \right)
\le \rev{4 e^{-K_w 2^{k} / 256}} .
\]
Because $\rev{e^{-K_w2^M/256}} \leq \frac{1}{2}$ for our choice of $M$ and $K_w \geq 4$, we can bound:
\begin{align*}
\sum_{k=M}^\infty \rev{e^{-K_w2^k/256}}
&= \sum_{j=0}^\infty (\rev{e^{-K_w2^{M}/256}})^{2^j}\\
&\le \sum_{j=0}^\infty (\rev{e^{-K_w2^{M}/256}})^{j+1}\\
&= \frac{\rev{e^{-K_w2^{M}/256}}}{1-\rev{e^{-K_w2^{M}/256}}}
\le 2\rev{e^{-K_w2^{M}/256}}.
\end{align*}
Since $\rev{2^M \ge 256 + 64 \log(8/\delta)}$, we obtain
\[
e^{K_w}\sum_{k = M}^\infty  4 \rev{e^{-K_w 2^k/256}}
\le 8 \rev{e^{K_w -K_w2^M/256}} \le  8 e^{-(K_w / 4) \log(8/\delta) } \le  8 e^{-\log(8/\delta) } = \delta.
\]

\noindent \emph{Case: \rev{$C_{w,\infty}(\epsilon_\mathcal W)>0$}.}
We distinguish two sub-cases.
Let $k_0$ be the largest integer such that
\[
      C_\Pcal^2 \rev{2^{k + 3}} \|w\|^2K_\rho \ge   m_\beta \|w\|_\infty  C_\Pcal \rev{C_{w,\infty}(\epsilon_\mathcal W)^{-1}} \rev{2^{(k + 3)/2}} r 2^{k - 1}.
\]
For $k\le k_0$, the same argument as above yields
\[
\mathbb{P}\left(  |\tilde{\mathbb{G}}_n^{w,m_\beta}(L_{h}-L_{h'})| > 2^{k-1} r^2   \right)
\le \rev{4 e^{ - K_w 2^k / 512}},
\]
and, since $\rev{2^M \ge 512 + 128 \log(16/\delta)}$,
\begin{align*}
e^{K_w}\sum_{k = M}^\infty  \rev{4 e^{-K_w 2^k/512}}
\le 8 \rev{e^{K_w -K_w2^M/512}} \le  8 e^{-(K_w / 4) \log(16/\delta) } \le  8 e^{-\log(16/\delta) } = \delta / 2.
\end{align*}
For $k\ge k_0$, we use $\min\{2,  C_\Pcal \rev{C_{w,\infty}(\epsilon_\mathcal W)^{-1}} \rev{2^{(k + 3)/2}} r\}   \le \min\{2,  C_\Pcal \rev{C_{w,\infty}(\epsilon_\mathcal W)^{-1}}  r\} \rev{2^{(k + 3)/2}} $ and \eqref{cond:rate} to upper-bound \eqref{eq:bernstein-intermediate} by
\[
4\exp\!\left(- \frac {K_w} {\rev{16}} \frac{ 2^{2(k-1)}}{ \rev{2^{(k + 3)/2 + k - 1}}}\right)
\le \rev{4 e^{- K_w2^{k / 2} / 96}}.
\]
Using a similar geometric series argument, we obtain
\begin{align*}
    e^{K_w}\sum_{k=M}^\infty \rev{4e^{-K_w 2^{k/2}/96}} \le 16 e^{K_w}  \rev{e^{- K_w 2^{M/2} /96}} .
\end{align*}
Because $\rev{2^M \ge  (96 +  24 \log(34 / \delta))^2}$, we get
\[
16 e^{K_w}  \rev{e^{- K_w 2^{M/2} /96}} \le 16 e^{-(K_w / 4) \log(34 / \delta)} \leq 16 e^{- \log(34 / \delta)} = \delta / 2.
\]
Combining the two sub-cases, we have
\begin{align*}
   &\quad \, e^{K_w}\sum_{k = M}^\infty \mathbb{P}\left(  |\tilde{\mathbb{G}}_n^{w,m_\beta}(L_{h}-L_{h'})| > 2^{k-1} r^2   \right) \\
   &\le \phantom{+}  e^{K_w}\sum_{k = M}^{\max\{M, k_0\} - 1} \mathbb{P}\left(  |\tilde{\mathbb{G}}_n^{w,m_\beta}(L_{h}-L_{h'})| > 2^{k-1} r^2   \right) \\
   &\quad \, +    e^{K_w}\sum_{k = \max\{M, k_0\}}^\infty \mathbb{P}\left(  |\tilde{\mathbb{G}}_n^{w,m_\beta}(L_{h}-L_{h'})| > 2^{k-1} r^2   \right)  \\
   &\le  e^{K_w}\sum_{k = M}^\infty  \rev{4 e^{-K_w 2^k/512}} + e^{K_w}\sum_{k = M}^\infty \rev{4e^{-K_w 2^{k/2} /96}} \\
   &\le \delta,
\end{align*}
\rev{This completes the proof.}

\end{proof}

\subsection{Explicit rate functions}\label{app:explicit-rates}

\begin{proposition}\label{prop:r-simplify}
Suppose $C_1,B_{\mathcal W}\lesssim1$ and there is $\alpha\in[0,2)$ such that
\begin{align}\label{cond:complexities}
\log N_\infty(\epsilon,\rev{\mathcal H_w(\epsilon_\mathcal W)})&\lesssim\|w\|^{-\alpha}\log(n/\epsilon), &
\log N_1(\epsilon,\mathcal W)&\lesssim\log(n/\epsilon).
\end{align}
\begin{enumerate}[label=(\roman*)]
\item If $C_{\beta,\rho}=C_\mathcal P^2K_\rho+m_\beta B_\mathcal W\le n$, there is $A\ge1$ such that condition~\eqref{cond:rate} holds with
\[
r(u)=u^{1-\alpha/2}\sqrt{A C_{\beta,\rho}\log n}.
\]
\rev{\item If $C_\mathcal P^2K_\rho\le n$ and $m_\beta B_\mathcal W C_\mathcal P C_{w,\infty}(\epsilon_\mathcal W)^{-1}\le n$ for every $w\in\mathcal W$, there is $A\ge1$ such that condition~\eqref{cond:rate} holds for any increasing $K$-Lipschitz $r$ satisfying, for every $w\in\mathcal W$,
\[
r(\|w\|)\ge \|w\|^{1-\alpha/2}\sqrt{A C_\mathcal P^2K_\rho\log n}
+A m_\beta B_\mathcal W C_\mathcal P C_{w,\infty}(\epsilon_\mathcal W)^{-1}\|w\|^{2-\alpha}\log n.
\]}
\end{enumerate}
\end{proposition}
The complexity of $\rev{\mathcal H_w(\epsilon_\mathcal W)}$ may therefore adapt to $n_{\mathrm{eff}}(w)=\|w\|^{-2}$ through $\alpha$. The assumptions exclude the unrestricted simplex of all weights but cover the structured families in Appendix~\ref{sec:appweights}.

\begin{proof}
Observe that $\|w\| \ge C_\mathcal{W} \ge n^{-1/2}$. Thus, there is a large enough constant $A$ such that, choosing $K = A^2 n^2$, we have $K_w \le A u^{-\alpha}\log n$. \rev{We will show that the rate in (i) and the sufficient bound in (ii) satisfy \eqref{cond:rate}; the Lipschitz property in (ii) is imposed explicitly.}
    \begin{enumerate}[label=(\roman*)]
        \item Condition \eqref{cond:rate} is satisfied provided
    \begin{align*}
        r(u)^2 \ge  A u^{2 - \alpha} C_{\beta, \rho} \log n,
    \end{align*}
    for some  large enough constant $A < \infty$.
    The function $r(u) = u^{1 - \alpha/2}\sqrt{A  C_{\beta, \rho} \log n}$ satisfies this with equality and has Lipschitz constant bounded by
    \begin{align*}
       \sup_{u \in [n^{-1/2}, C_1]} |r'(u)| &= (1- \alpha/2)  \sup_{u \in [n^{-1/2}, C_1]}u^{- \alpha/2} \sqrt{A C_{\beta, \rho} \log n} \\
       &\le  n^{\alpha /4 }  (1- \alpha/2) \sqrt{An \log n} \\
         &\le  n^{3/4} \sqrt{A \log n} \\
        &\le An \le A^2n^2,
    \end{align*}
    as desired.

    \rev{\item Condition \eqref{cond:rate} is satisfied provided, for every $w\in\mathcal W$,
    \begin{align*}
        r(\|w\|)^2 \ge  A \|w\|^{2 - \alpha} \bigl(C_\Pcal^2 K_\rho + m_\beta B_\mathcal W C_\mathcal P C_{w,\infty}(\epsilon_\mathcal W)^{-1}r(\|w\|)\bigr) \log n.
    \end{align*}
    The lower bound in (ii) implies this inequality directly.}
    \end{enumerate}
\end{proof}

\subsection{Proof of the lower bound}\label{sec:lower-bounds}

\begin{proof}[Proof of Proposition~\ref{prop:weighted-fano-lower}]
Partition the sample into $K$ consecutive blocks $I_k$ of size $M$ and each $I_k$ into $B$ subblocks $J_{k,b}$ of size $m$.
For $i\in J_{k,b}$, let $Y_i=\mu_k+\xi_{k,b}$, where $\xi_{k,b}\stackrel{\mathrm{iid}}\sim N(0,I_d)$.
Let $\Theta_{\mathrm{ns}}$ contain the nonconstant mean sequences in $\mathcal H^K$, $\mathcal H=[-1,1]^d$, and set $\mathcal W=\{w^{(k)}:w_i^{(k)}=M^{-1}\mathbbm1\{i\in I_k\}\}$.
Under square loss, $\mu_{w^{(k)}}^*=\mu_k$ and the weighted excess risk is $\|\widehat\mu_{w^{(k)}}-\mu_k\|_2^2$. Moreover, $\|w^{(k)}\|_2^2=M^{-1}$ and $n_{\mathrm{eff}}(w^{(k)})=M$, whereas repeated observations within a subblock carry exactly the information in one Gaussian vector. Thus each block provides only $B=M/m$ independent observations from $N(\mu_k,I_d)$.

A Gilbert--Varshamov packing supplies $c^0,\ldots,c^J\in\{-1,1\}^d$ with $\log(J+1)\asymp d$ and pairwise Hamming distance at least $d/8$. Set
\[
a=c_0\sqrt{\frac{\log(KJ)}{Bd}},\qquad \nu^j=ac^j.
\]
The condition $d+\log K\le Bd$ ensures that $\nu^j\in\mathcal H$ for a sufficiently small universal $c_0$.

Let $P^{(0)}$ assign mean $\nu^0$ to every block. For $1\le k\le K$ and $1\le j\le J$, let $P^{(k,j)}$ assign mean $\nu^j$ to $I_k$ and $\nu^0$ to all other blocks. Every alternative has nonconstant means because $K\ge2$, so it belongs to $\Theta_{\mathrm{ns}}$; $P^{(0)}$ is used only as a reference distribution. Only the $B$ independent Gaussian values on $I_k$ change, giving
\[
\operatorname{KL}(P^{(k,j)},P^{(0)})
=\frac B2\|\nu^j-\nu^0\|_2^2
\lesssim Ba^2d
\lesssim c_0^2\log(KJ).
\]
The collections of weighted oracles for distinct alternatives are separated, in the supremum metric over $\mathcal W$, by a constant multiple of $a\sqrt d$. Fano's inequality therefore gives
\[
\inf_{(\widehat\mu_w)}\sup_{\mu\in\Theta_{\mathrm{ns}}}
\mathbb E_\mu\!\left[\sup_{w\in\mathcal W}\|\widehat\mu_w-\mu_w^*\|_2^2\right]
\gtrsim a^2d\asymp\frac{d+\log K}{B}.
\]
When $KJ$ is bounded, the same conclusion follows by testing two alternatives: their Kullback--Leibler divergence is bounded by a constant multiple of $c_0^2$, while their squared separation remains of order $a^2d$.

Finally, the weight vectors have pairwise $\ell_1$ distance $2$, so $N_1(1/2,\mathcal W)=K$, while $\log N_\infty(1/2,\mathcal H)\asymp d$. Substituting $B=M/m$ proves the stated bound. For this process, $\rho(\ell)=1$ for $1\le\ell<m$ and $\rho(\ell)=0$ for $\ell\ge m$, hence $K_\rho=2m-1$. Together with $\|w\|_2^2=1/M$, this gives the claimed dependence scaling. Setting $m=1$ in the construction and proof yields the independent nonstationary special case.
\end{proof}
\subsection{Proof of Example \texorpdfstring{\ref{exp:assumptions}}{\ref*{exp:assumptions}}}

First, under square loss and for any distribution $P$ on $(Y,Z)$, let $h_P^*(z)=\mathbb{E}_P[Y\mid Z=z]$. Then, for every measurable $h$,
\begin{align*}
\mathbb{E}_P[(Y-h(Z))^2-(Y-h_P^*(Z))^2]
&= \mathbb{E}_P[(h(Z)-h_P^*(Z))^2] \\
&\quad + 2\,\mathbb{E}_P[(h_P^*(Z)-Y)(h(Z)-h_P^*(Z))] \\
&= \|h-h_P^*\|_{L_2(P)}^2,
\end{align*}
since $\mathbb{E}_P[h_P^*(Z)-Y\mid Z]=0$.
Applying this with $P=\sum_{t=1}^n w_t P_t$ shows that \eqref{cond:bernstein} holds with equality.

\textit{Scalar hypothesis class.} Let $\mathcal H=\{h(z)=c:c\in[-B,B]\}$. For $h(z)=c$ and $h'(z)=c'$, $\|h-h'\|_{L_2(P)}=|c-c'|$ and $\|h-h'\|_\infty=|c-c'|$ for every $P\in\mathcal P$. Therefore
\[
\sup_{P\in\mathcal P}\|h-h'\|_{L_2(P)}
=
\inf_{P\in\mathcal P}\|h-h'\|_{L_2(P)}, \qquad \inf_{P\in\mathcal P}\|h-h'\|_{L_2(P)}=\|h-h'\|_\infty,
\]
so \eqref{cond:loss-upper} holds with $C_{\mathcal P}=1$, and \eqref{cond:infty} with \rev{$C_{w,\infty}(\epsilon)=1$}.

\textit{Linear feature class.} Let $\mathcal H=\{h(z)=\beta^\top\Phi(z):\|\beta\|_2\le B\}$ and define $\Sigma_P=\mathbb{E}_P[\Phi(Z)\Phi(Z)^\top]$. For $h(z)=\beta^\top\Phi(z)$ and $h'(z)=\beta'^\top\Phi(z)$, writing $\Delta=\beta-\beta'$, we have $\|h-h'\|_{L_2(P)}^2=\Delta^\top \Sigma_P \Delta$. If the Gram matrices are uniformly well-conditioned, i.e.,
\[
0<\underline\lambda
\le \lambda_{\min}(\Sigma_P)
\le \lambda_{\max}(\Sigma_P)
\le \bar\lambda<\infty
\qquad\text{for all }P\in\mathcal P,
\]
then
\[
\underline\lambda \|\Delta\|_2^2
\le \|h-h'\|_{L_2(P)}^2
\le \bar\lambda \|\Delta\|_2^2.
\]
Taking the supremum and infimum over $P\in\mathcal P$ yields
\[
\sup_{P\in\mathcal P}\|h-h'\|_{L_2(P)}
\le \sqrt{\bar\lambda/\underline\lambda}
\inf_{P\in\mathcal P}\|h-h'\|_{L_2(P)}.
\]
Hence \eqref{cond:loss-upper} holds with $C_{\mathcal P}=\sqrt{\bar\lambda/\underline\lambda}$. If moreover $\sup_{z\in\mathcal Z}\|\Phi(z)\|_2\le 1$, then
\[
\|h-h'\|_\infty
= \sup_{z\in\mathcal Z} |\Delta^\top\Phi(z)|
\le \|\Delta\|_2,
\]
and therefore it is easy to see that \eqref{cond:infty} holds with \rev{$C_{w,\infty}(\epsilon)=\sqrt{\underline\lambda}$}:
\[
\inf_{P\in\mathcal P}\|h-h'\|_{L_2(P)}
\ge \sqrt{\underline\lambda}\|\Delta\|_2
\ge \sqrt{\underline\lambda}\|h-h'\|_\infty.
\]

\subsection{Proof of Lemma \ref{lem:driftlem}}
\label{sec:proof-driftlem}

Write $p_i=p_{Z_i}$ and $h_i^*=h_{P_i}^*$. Under square loss, the oracle under the weighted past distribution is given pointwise by the conditional mean under the weighted mixture. Hence, for all $z$ such that $\sum_{j=1}^t w_jp_j(z)>0$,
\[
    h_w^*(z)
    =
    \frac{\sum_{i=1}^t w_i p_i(z)h_i^*(z)}
    {\sum_{j=1}^t w_j p_j(z)}.
\]
For ease of notation, we omit $z$ in the following. By the triangle inequality,
\begin{align*}
    \|h_w^*-h_{P_{t+1}}^*\|_{L_2(P_{t+1})}
    &=
    \left\|
    \sum_{i=1}^t
    \frac{w_i p_i}{\sum_{j=1}^t w_jp_j}
    \bigl(h_i^*-h_{P_{t+1}}^*\bigr)
    \right\|_{L_2(P_{t+1})} \\
    &\le
    \sum_{i=1}^t
    \left\|
    \frac{w_i p_i}{\sum_{j=1}^t w_jp_j}
    \bigl(h_i^*-h_{P_{t+1}}^*\bigr)
    \right\|_{L_2(P_{t+1})} \\
    &\le
    \sum_{i=1}^t
    \left\|
    C_p w_i
    \bigl(h_i^*-h_{P_{t+1}}^*\bigr)
    \right\|_{L_2(P_{t+1})} \\
    &=
    C_p
    \sum_{i=1}^t
    w_i
    \|h_i^*-h_{P_{t+1}}^*\|_{L_2(P_{t+1})}.
\end{align*}

\subsection{Proof of Example \ref{exp:smoothdrift}}

For exponential weights, set \(\rho=\exp(-\theta)\) and \(k=t-i\). Then
\begin{align*}
    \sum_{i=1}^t w_i(t+1-i)^\nu
    &=
    \frac{1-\rho}{1-\rho^t}
    \sum_{k=0}^{t-1}\rho^k(k+1)^\nu .
\end{align*}
Since the weights sum to one and \(t+1-i\le t\), we have $ \sum_{i=1}^t w_i(t+1-i)^\nu \le t^\nu $.
If \(t(1-\rho)\ge 1\), then \(1-\rho^t\) is bounded away from zero, and hence
\[
    \sum_{i=1}^t w_i(t+1-i)^\nu
    =
    \frac{1-\rho}{1-\rho^t}
    \sum_{k=0}^{t-1}\rho^k(k+1)^\nu
    \lesssim
    (1-\rho)
    \sum_{k=0}^{\infty}\rho^k(k+1)^\nu
    \lesssim
    (1-\rho)^{-\nu}.
\]
Thus, in all cases,
\[
    \sum_{i=1}^t w_i(t+1-i)^\nu
    \lesssim
    \min\{t,(1-\rho)^{-1}\}^{\nu}.
\]
Moreover, the computation in Lemma~\ref{lem:exponentialassump} gives:
\[
    \|w^{(t)}(\theta)\|^{-2}
    =
    \frac{1+\rho}{1-\rho}\cdot
    \frac{1-\rho^t}{1+\rho^t}
    \geq \frac{1-\rho^t}{2(1-\rho)}
    \gtrsim \min\{t,(1-\rho)^{-1}\}.
\]
Furthermore, $    \|w^{(t)}(\theta)\|^{-2}
    \le
    \frac{2}{1-\rho}$, and consequently,
\[
    \sum_{i=1}^t w_i(t+1-i)^\nu
    \lesssim
    \|w^{(t)}(\theta)\|^{-2\nu}
    \lesssim
    (1-\rho)^{-\nu}
    =
    (1-e^{-\theta})^{-\nu}.
\]
Combining this with Lemma~\ref{lem:driftlem} yields
\[
   \|h_w^*-h_{P_{t+1}}^*\|_{L_2(P_{t+1})}
   \lesssim
   C_p\kappa \|w^{(t)}(\theta)\|^{-2\nu}
   \lesssim
   C_p\kappa (1-e^{-\theta})^{-\nu}.
\]

\subsection{Time-uniform Bound}\label{app:time-uniform}
Let $n_0 \in \{1, \dots, n\}$ and suppose that at every time $t = n_0, \dots, n$, we pick a weight $w^{(t)} \in \mathcal{W}$ such that $w_i^{(t)}=0$ for all $i>t$ and compute the weighted ERM
    $\hat h_{w^{(t)}} = \argmin_{h \in \mathcal{H}_{w^{(t)}}} R_n^{w^{(t)}}(h).$

\begin{theorem}\label{thm:fastsum}
Under the conditions of Theorem~\ref{thm:fast}, with probability at least $1-2\delta$, every sequence $w^{(n_0)},\ldots,w^{(n)}\in\mathcal W$ satisfies
\begin{align*}
\sum_{t=n_0}^n \mathbb E[ L(X_{t+1},\hat h_{w^{(t)}})-L(X_{t+1},h^*_{P_{t+1}})]
&\lesssim \sum_{t=n_0}^n r(\|w^{(t)}\|)^2\log^2(1/\delta)\\
&\quad+\sum_{t=n_0}^n\|h^*_{w^{(t)}}-h^*_{P_{t+1}}\|_{L_2(P_{t+1})}^2.
\end{align*}
\end{theorem}

\begin{proof}
By Theorem \ref{thm:fast}, with probability at least $1-2\delta$,
\begin{align} \label{eq:rate-proof}
    \sup_{w \in \mathcal{W}} r(\|w\|)^{-2} \|\hat h_{w} -  h_{w}^*\|_{L_2(\mathcal{P})}^2 \lesssim \log^2(1/\delta).
\end{align}
Thus, on the same event,
    \begin{align*}
     &\quad \, \sum_{t=n_0}^n\left\{\mathbb{E}[L(X_{t + 1}, \hat h_{w^{(t)}})] - \mathbb{E}[L(X_{t + 1}, h_{P_{t + 1}}^{*})]\right\} \\
    &\lesssim  \sum_{t=n_0}^n \left\{\| \hat h_{w^{(t)}} - h_{w^{(t)}}^{*} \|_{L_2(P_{t + 1})}^2 +  \|h_{w^{(t)}}^{*} -  h_{P_{t + 1}}^{*}\|_{L_2(P_{t + 1})}^2\right\} \\
      &\lesssim  \sum_{t=n_0}^n \left\{r(\|w^{(t)}\|)^2\log^2(1/\delta) +  \|h_{w^{(t)}}^{*} -  h_{P_{t + 1}}^{*}\|_{L_2(P_{t + 1})}^2\right\},
\end{align*}
where we used \eqref{eq:decomposition} in the first inequality and \eqref{eq:rate-proof} in the second.
\end{proof}

\section{Derivations for the applications}\label{sec:appapplication}

\subsection{Linear models}\label{app:linear}

\rev{Here $\mathcal H_w$ is fixed, so $\mathcal H_w(\epsilon_\mathcal W)=\mathcal H$.}

For any $h, h' \in \mathcal{H}$ with $h(z) = \beta^\top z$ and $h'(z) = \beta'^\top z$, we have
\begin{align*}
    \inf_{P \in \mathcal{P}} \| h - h'\|_{L^2(P)}^2  = \inf_{P \in \mathcal{P}}(\beta - \beta')^\top \mathbb{E}_P[ZZ^\top] (\beta - \beta')
    &\ge \inf_{P\in \mathcal{P}} \lambda_{\min}(\mathbb{E}_P[ZZ^\top]) \|\beta - \beta'\|^2,
\end{align*}
and, by the Cauchy--Schwarz inequality,
\begin{align*}
     \|\beta - \beta'\|^2 \ge \|\beta - \beta'\|^2 \sup_{z \in \Zcal} \|z\|^2 \ge \sup_{z \in \Zcal}|\beta^\top z - \beta'^\top z|^2.
     \end{align*}
 Therefore, \eqref{cond:infty} is satisfied with \rev{$C_{w,\infty}(\epsilon_\mathcal W)=\inf_{P\in \mathcal{P}} \sqrt{\lambda_{\min}(\mathbb{E}_P[ZZ^\top])}$}. For the covering number, apply Lemma \ref{lem:covradius} to obtain
\begin{align*}
   \log  N_\infty(\epsilon,\mathcal H)
    \le
    \log  \left(\frac{3B}{\epsilon}\right)^p \lesssim p \log (1/\epsilon).
\end{align*}
 Thus, \eqref{cond:complexities} holds with $\alpha=0$. Applying Proposition~\ref{prop:r-simplify} (ii) gives
\[
    r(u) = u\sqrt{A C_\mathcal P^2K_\rho p\log n} +   A C_{\beta,\infty}u^2 p\log n,
\]
where \rev{$C_{\beta,\infty} =m_\beta B_\mathcal W C_\mathcal P C_{w,\infty}(\epsilon_\mathcal W)^{-1}$}. Since $C_\mathcal P, B_\mathcal W, K_\rho \lesssim 1$ by assumption, we have
\[
    r(u)^2
    \lesssim
    p u^2 \log n
    +
    p^2 m_\beta^2 \rev{C_{w,\infty}(\epsilon_\mathcal W)^{-2}}u^4 \log^2 n.
\]
Combining the preceding display with Theorem \ref{thm:fast} and setting $u=\|w\|$ gives the final rate:
\[
    \|\hat h_w-h_w^*\|_{L_2(\mathcal P)}^2
    \lesssim
    p\|w\|^2\log n
    +
    p^2\|w\|^4m_\beta^2\rev{C_{w,\infty}(\epsilon_\mathcal W)^{-2}}\log^2 n.
\]
For unweighted ERM, we have $\|w\|^2=1/n$. For any fixed $p$,
\begin{align*}
    \|\hat h_w-h_w^*\|_{L_2(\mathcal P)}^2
    &\lesssim
    \frac{p\log n}{n}
    +
    \frac{p^2m_\beta^2\log^2 n}{n^2} \lesssim
    \frac{p\log n}{n}
    +
    \frac{p^2\log^4 n}{n^2}
    \lesssim
    \frac{p\log n}{n}.
\end{align*}

\subsection{Basis expansions}\label{app:basis}

We consider the construction from Section~\ref{sec:applications}. \rev{Take the dyadically rounded dimension
\[
q(w)=2^{\left\lceil\log_2\left\lceil\|w\|^{-2/3}\right\rceil\right\rceil},
\]
so $q(w)\asymp\|w\|^{-2/3}$ and the corresponding equal-width partitions are nested.} Define $\phi_{j,w}(z)=\mathbbm{1}\{z\in [\frac{j-1}{q(w)},\frac{j}{q(w)})\}, j=1,\dots,q(w)$. By $1$-Lipschitzness,
\begin{align*}
    \inf_{h\in\mathcal H_w}\|h-h_w^*\|_\infty
    &\le
    \sup_{j=1,\dots,q(w)}\sup_{z\in [\frac{j-1}{q(w)},\frac{j}{q(w)})}
    |h_w^*((j-1)/q(w))-h_w^*(z)|
    \le q(w)^{-1}.
\end{align*}
Thus, the squared approximation error is bounded by $q(w)^{-2}$. \rev{If $\|v-w\|_1,\|v'-w\|_1\le\epsilon_\mathcal W$, then $|\|v\|-\|w\||,|\|v'\|-\|w\||\le\epsilon_\mathcal W\le\|w\|/64$, and hence $q(v),q(v')\asymp q(w)$. For $h\in\mathcal H_v$ and $h'\in\mathcal H_{v'}$, their difference is piecewise constant on the common refinement. Because the dyadic partitions are nested, this refinement has $q_{\max}=\max\{q(v),q(v')\}\asymp q(w)$ equal-width intervals. Under the uniform design,
\[
\|h-h'\|_{L_2(P)}^2
=\frac1{q_{\max}}\sum_{j=1}^{q_{\max}}(h_j-h'_j)^2
\ge \frac1{q_{\max}}\|h-h'\|_\infty^2
\gtrsim q(w)^{-1}\|h-h'\|_\infty^2,
\]
so $C_{w,\infty}(\epsilon_\mathcal W)\asymp q(w)^{-1/2}$.} By Lemma \ref{lem:covradius},
\[
   \log  N_\infty(\epsilon,\mathcal H_w)
    \le
    \log \left(\frac{3B}{\epsilon}\right)^{q(w)} \lesssim q(w) \log(1/\epsilon).
\]
\rev{The union bound
$N_\infty(\epsilon,\bigcup_j\mathcal H_j)\le\sum_jN_\infty(\epsilon,\mathcal H_j)$ over the $O(1)$ dyadic dimensions in the local neighborhood therefore gives
$\log N_\infty(\epsilon,\mathcal H_w(\epsilon_\mathcal W))\lesssim q(w)\log(n/\epsilon)$.}
Then $\rev{\log N_\infty(\epsilon,\mathcal H_w(\epsilon_\mathcal W))\lesssim \|w\|^{-2/3}\log(n/\epsilon)}$, so \eqref{cond:complexities} holds with $\alpha=2/3$. Applying Proposition~\ref{prop:r-simplify} (ii) with \rev{$C_{w,\infty}(\epsilon_\mathcal W)^{-1}\lesssim q(w)^{1/2}\lesssim \|w\|^{-1/3}$} and $u=\|w\|$,
\[
    r(\|w\|)^2
    \lesssim
    \|w\|^{4/3}\log n
    +
    \|w\|^2m_\beta^2\log^2 n.
\]
Combining this with $q(w)^{-2}\lesssim \|w\|^{4/3}$ and Theorem~\ref{thm:fast} yields
\[
    \|\hat h_w-h_w^*\|_{L_2(\mathcal P)}^2
    \lesssim
    \|w\|^{4/3}\log n
    +
    \|w\|^2m_\beta^2\log^2 n.
\]
For unweighted ERM, $\|w\|^2=1/n$ and thus
\begin{align*}
    \|\hat h_w-h_w^*\|_{L_2(\mathcal P)}^2
    &\lesssim
    \frac{\log n}{n^{2/3}}
    +
    \frac{m_\beta^2\log^2 n}{n} \lesssim
    \frac{\log n}{n^{2/3}}
    +
    \frac{\log^4 n}{n}
    \lesssim
    \frac{\log n}{n^{2/3}}.
\end{align*}

\subsection{Neural networks}\label{app:relu}

Theorem 3.1 of \citet{nagler2026optimal} shows that for $b_w \gtrsim \nu_w$,
\begin{align*}
    \inf_{h \in \mathcal{H}_w} \|h - h_{w}^*\|_\infty & \lesssim (\nu_w \ell_w /\sqrt{\log \ell_w})^{-2a/p}.
\end{align*}
Further, Theorem 2.1 of \citet{ou2024covering} yields
\begin{align*}
    \log N_\infty(\epsilon, \mathcal{H}_w) \lesssim \nu_w^2 \ell_w^2 \log(\nu_w b_w / \epsilon).
\end{align*}
Noting $\log (\nu_w b_w)  \lesssim \log n$, this simplifies to
\begin{align*}
    \inf_{h \in \mathcal{H}_w} \|h - h_{w}^*\|_\infty & \lesssim \|w\|^{2a / (2a + p)}  \log^{a/p} (n), \quad     \log N_\infty(\epsilon, \mathcal{H}_w) \lesssim \|w\|^{- 2p / (2a + p)} \log(n / \epsilon).
\end{align*}
\rev{For $\|v-w\|_1\le\epsilon_\mathcal W$, the architecture sizes $\nu_v\ell_v$ and $\nu_w\ell_w$ are comparable. Since $\nu_v$ and $\ell_v$ are positive integers of polynomially bounded size, at most $n^{O(1)}$ distinct architecture pairs occur in this local neighborhood. Taking the union of the covers over these distinct architectures therefore adds only $O(\log n)$ to the logarithm of the covering number, which is absorbed by the displayed bound; hence $\log N_\infty(\epsilon,\mathcal H_w(\epsilon_\mathcal W))\lesssim\|w\|^{-2p/(2a+p)}\log(n/\epsilon)$.}
\rev{Use bounded, clipped outputs and $C_{w,\infty}(\epsilon_\mathcal W)=0$.}
Proposition \ref{prop:r-simplify} (i) and Theorem \ref{thm:fast} yield
\begin{align*}
    \|\hat h_w - h_w^*\|_{L_2(\mathcal{P})}^2 & \lesssim \|w\|^{4a/(2a + p)} \log^{2a/p} (n) + \|w\|^{2 - 2p / (2a + p)} \rev{\log^2 n} \lesssim  \|w\|^{4a/(2a + p)}\log^{\zeta}  (n),
\end{align*}
with $\zeta = \rev{\max\{2a/p, 2\}}$.

\subsection{Structural breaks}\label{app:structural-break-proof}

Let $w(s)$ denote the uniform weights on the most recent $s$ observations. Let
\rev{$\tau_k$} be the most recent break point, so the current regime has age
\rev{$d=t-\tau_k$}, and suppose the window intersects at most the current and the
preceding regime. Define the break size by
\[
\Delta
=
\rev{\|h^*_{P_{\tau_k}}-h^*_{P_{t+1}}\|_{L_2(P_{t+1})}}.
\]
Exactly $(s-d)_+$ observations in the window come from before the break.
Lemma~\ref{lem:driftlem} therefore gives
\[
\|h_{w(s)}^*-h^*_{P_{t+1}}\|_{L_2(P_{t+1})}
\le
C_p\Delta\frac{(s-d)_+}{s}
=
C_p\Delta\left(1-\frac{d}{s}\right)_+.
\]

\rev{Since}
$\|w(s)\|_2^2=1/s$, the leading learning terms derived in
Appendices~\ref{app:linear}--\ref{app:relu} are $c s^{-b}$, where
\[
\rev{(c,b)=(p\log n,1),\qquad (c,b)=(\log n,2/3),\qquad
(c,b)=(\log^\zeta n,\varrho)}
\]
for linear models, basis expansions, and ReLU networks, respectively, with
$\varrho=2a/(2a+p)$. Corollary~\ref{cor:excess-risk} thus yields the direct
bound
\begin{equation}\label{eq:structural-break-bound}
\mathbb E[L(X_{t+1},\hat h_{w(s)})]
-\mathbb E[L(X_{t+1},h^*_{P_{t+1}})]
\lesssim
\frac{c}{s^b}
+
C_p^2\Delta^2\left(1-\frac{d}{s}\right)_+^2.
\end{equation}

For $s\le d$, the drift term vanishes and the learning term decreases with
$s$. Hence $s=d$ is the best window that uses only observations from the
current regime. If $C_p\Delta=0$, the bound is decreasing on the full range and
$s=n$ is optimal. Otherwise, to see when older observations improve the bound,
consider $s\ge d$ and write
\[
\epsilon=1-\frac{d}{s},
\qquad
s=\frac{d}{1-\epsilon},
\qquad
0\le\epsilon\le1-\frac{d}{n}.
\]
The right-hand side of~\eqref{eq:structural-break-bound} becomes
\[
\frac{c}{d^b}(1-\epsilon)^b+C_p^2\Delta^2\epsilon^2.
\]
Because $b\in(0,1]$, concavity gives
$(1-\epsilon)^b\le1-b\epsilon$. It follows that
\[
\frac{c}{d^b}(1-\epsilon)^b+C_p^2\Delta^2\epsilon^2
\le
\frac{c}{d^b}-\frac{bc}{d^b}\epsilon+C_p^2\Delta^2\epsilon^2.
\]
The quadratic on the right is minimized at
$\epsilon_0=bc/(2C_p^2\Delta^2d^b)$. Accounting for the largest available
window, \rev{its feasible optimizer is}
\[
\rev{\epsilon^*
=\min\left\{\frac{bc}{2C_p^2\Delta^2d^b},1-\frac{d}{n}\right\}},
\qquad
s^*=\frac{d}{1-\epsilon^*}.
\]
\rev{If $bc/(2C_p^2\Delta^2d^b)\le1-d/n$, substitution into the upper bound gives}
\[
\mathbb E[L(X_{t+1},\hat h_{w(s^*)})]
-\mathbb E[L(X_{t+1},h^*_{P_{t+1}})]
\lesssim
\rev{\frac{c}{d^b}
\left(1-\frac{b^2c}{4C_p^2\Delta^2d^b}\right)}.
\]
\rev{Otherwise the optimizer of the upper bound is attained at the boundary $s^*=n$, where the generic bound equals
\[
cn^{-b}+C_p^2\Delta^2\left(1-\frac dn\right)^2.
\]
}
\rev{Substituting the three pairs $(c,b)$ gives the structural-break column of
Table~\ref{tab:rates}.}

\section{Experimental setup}\label{sec:simulations}

All three panels of Figure~\ref{fig:drift-simulations} use the scalar location
model
\[
Y_t=\mu_t+\eta_t,
\qquad
\hat\mu_w=\sum_{i=1}^n w_iY_i,
\]
where the mean path is deterministic and independent of the noise. The target
is the next mean $\mu_{n+1}$.

\paragraph{Panel specifications.}
\begin{itemize}[leftmargin=1.5em]
    \item \emph{Panel (a): learning rates.} We use uniform weights on all $n$
    observations, set $\mu_t=0$, and compare
    $\mathbb E[(n^{-1}\sum_{i=1}^n\eta_i)^2]$ under iid and dependent noise as
    $n$ increases.

    \item \emph{Panel (b): structural break.} We set $n=2048$ and
    \[
    \mu_t=
    \begin{cases}
    -0.05, & t\le n-256,\\
    \phantom{-}0.05, & t>n-256,
    \end{cases}
    \qquad
    \mu_{n+1}=0.05.
    \]
    The estimator averages the most recent $s$ observations, and the horizontal
    axis reports the window length $s$.

    \item \emph{Panel (c): smooth drift.} We again set $n=2048$, now with
    $\mu_t=0.5t/(n+1)$ for $t=1,\ldots,n+1$. The normalized exponential weights
    satisfy
    \[
    w_i(\theta)\propto \exp[-\theta(n-i)],
    \]
    and the horizontal axis gives their decay horizon $1/\theta$.
\end{itemize}

\paragraph{Noise processes.}
For the dependent process, let the consecutive block lengths be independent
copies of
\[
L=5\lfloor U^{-1/\gamma}\rfloor,
\qquad
U\sim\mathrm{Unif}(0,1),
\qquad
\gamma=3.5.
\]
Thus $\Pr(L\ge5\ell)=\ell^{-\gamma}$. Each block receives an independent
Rademacher sign, and $\eta_t$ is held fixed at $0.35$ times that sign throughout
the block. We initialize the process in equilibrium: the block containing the
first observation is drawn from the length-biased distribution
\[
\Pr(L_0=l)=\frac{l\Pr(L=l)}{\mathbb E[L]},
\]
and its starting position is uniform within the block. For numerical sampling,
the support of $L_0/5$ is truncated at $10^5$; all subsequent blocks follow the
ordinary length distribution. The resulting autocovariance is
\[
\operatorname{Cov}(\eta_t,\eta_{t+k})
=
0.35^2\frac{\mathbb E[(L-k)_+]}{\mathbb E[L]},
\]
which decays as $k^{-(\gamma-1)}$ and is summable. In the iid benchmark, the
Rademacher sign is instead redrawn independently at every time point.

\paragraph{Risk curves.}
For either weight family, the expected excess risk separates exactly into
\[
\underbrace{\sum_{i,j=1}^n
w_iw_j\operatorname{Cov}(\eta_i,\eta_j)}_{\text{learning error}}
+
\underbrace{\left(\sum_{i=1}^n w_i\mu_i-\mu_{n+1}\right)^2}_{\text{drift error}}.
\]
We compute both terms from the known mean path and autocovariance. We also
validate the resulting excess-risk curves using 1,200 independent replications.
The vertical dotted lines indicate the minimizers of the computed curves. Panels
(b)--(c) therefore display directly how a longer averaging horizon reduces the
learning error while increasing the drift error.
\section{Weight families}\label{sec:appweights}

\begin{lemma}\label{lem:covradius}
    Let $\|\cdot\|$ be a norm on $\mathbb{R}^d$ and $S=\left\{x\in \mathbb{R}^d: \|x\| \leq r \right\}$. Then
    \begin{equation*}
        N(\epsilon, S, \| \cdot \|) \leq \rev{\left(1+\frac{2r}{\epsilon}\right)^d}.
    \end{equation*}
\end{lemma}
\begin{proof}
    The result follows from \citet{vershynin2018high}, Proposition 4.2.10.
\end{proof}
In the following, we consider $n$-dimensional weight vectors $w^{(t)} \in \mathbb{R}^n$, for $t \in \{1,\dots,n\}$, whose entries satisfy $w_i \in \mathbb{R}$ if $i\leq t$ and $w_i=0$ if $i>t$.

\subsection{Uniform window weights}

\begin{lemma}\label{lem:windowweightsassum}
Let $s\in\{1,\dots,t\}$ and $t \in \{1,\dots,n\}$. Define the weight vector $w^{(t)}(s)\in\mathbb{R}^n$ by
\begin{equation*}
    w_i^{(t)}(s) = \mathbbm{1}\{i \leq t\} \times
\begin{cases}
1/s, & i\in\{t-s+1,\dots,t\},\\
0, & i\in\{1,\dots,t-s\}.
\end{cases}
\end{equation*}
Let $\mathcal{W}_t = \{w^{(t)}(s) : s= 1,\dots,t\}$. Then
\[
B_{\mathcal{W}}
= 1, \qquad C_1 = 1, \qquad N_1(\epsilon,\mathcal{W}_t) \leq t, \qquad N_1\left(\epsilon,\bigcup_{t=1}^n \mathcal{W}_t\right) \leq \frac{n(n+1)}{2}.
\]
\end{lemma}

\begin{proof}
    Fix $t \in \{1, \dots, n\}$ and $s\in\{1,\dots,t\}$. For simplicity, we abbreviate $w=w^{(t)}(s)\in\mathcal{W}_t$ in the following. By construction, $w$ has exactly $s$ nonzero entries, each equal to $1/s$. Hence $\|w\|_\infty = 1/s$.
    Moreover,
    \begin{align*}
        \|w\|^2 = \sum_{i=1}^t w_i^2 = s\cdot \left(\frac{1}{s}\right)^2 = \frac{1}{s}.
    \end{align*}
    The ratio is then $\|w\|_\infty / \|w\|^2 = 1$, which is independent of $t$ and $s$. Therefore $B_{\mathcal{W}}=1$. The second claim follows from $\|w\|_1=s(1/s)=1$, so $C_1=1$. For the covering number, $|\mathcal{W}_t|=t$, and thus we can cover the entire space with at most $t$ balls for any $\epsilon>0$, i.e., $N_1(\epsilon,\mathcal{W}_t)\leq t$. The final statement follows from $|\bigcup_{t=1}^n \mathcal{W}_t|\leq \sum_{t=1}^n |\mathcal{W}_t|=\sum_{t=1}^n t=n(n+1)/2$.
\end{proof}

\subsection{Exponential weights}
\begin{lemma}\label{lem:expweights-cover}
Fix $t\in\{1,\dots,n\}$. For $\theta \in (0,\infty)$ define the weight vector $w^{(t)}(\theta)  \in \mathbb{R}^n$, where for $i=1,\dots,t$:
    \begin{equation}\label{eq:expweights}
        w_{i}^{(t)}(\theta)= \mathbbm{1}\{i \leq t\} \frac{\exp(-\theta(t-i))}{\sum_{j=1}^t \exp(-\theta(t-j))}.
    \end{equation}
Let $\Theta =(0,R)$ with $R < \infty$ and define the class
$\mathcal{W}_t(\Theta) = \{ w^{(t)}(\theta) : \theta\in\Theta\}$.
Then for every $\epsilon>0$,
\[
N_1(\epsilon,\mathcal{W}_t(\Theta))
\le
\frac{3R(t-1)}{\epsilon}\rev{\quad(t\ge2;\ t=1\text{ is trivial})},
\qquad N_1\left(\epsilon, \bigcup_{t = 1}^n \mathcal{W}_t(\Theta)\right)
\le
\frac{3Rn^2}{2\epsilon}.
\]
\end{lemma}

\begin{proof}
Fix $t \in \{1, \dots, n\}$ and $\theta \in (0, \infty)$. We again abbreviate $w(\theta) = w^{(t)}(\theta)$ in what follows.
Let $x_i = (t-i)$ for $i=1,\dots,t$ and notice $x_i\in[0,(t-1)]$. Differentiating gives
\begin{align*}
    \frac{d}{d\theta} w_{i}(\theta)
&= \frac{-x_i e^{-\theta x_i}}{\sum_{j=1}^t e^{-\theta x_j}}
-
\frac{e^{-\theta x_i}}{(\sum_{j=1}^t e^{-\theta x_j})^2} \left(-\sum_{j=1}^t x_je^{-\theta x_j}\right) \\
&= \frac{-x_i e^{-\theta x_i}}{\sum_{j=1}^t e^{-\theta x_j}}
-
\frac{e^{-\theta x_i}}{(\sum_{j=1}^t e^{-\theta x_j})^2} \left(-\left(\sum_{j=1}^t e^{-\theta x_j}\right)\sum_{j=1}^t w_{j}(\theta)x_j\right) \\
&= -w_{i}(\theta)x_i
+
\frac{e^{-\theta x_i}}{\sum_{j=1}^t e^{-\theta x_j}} \sum_{j=1}^t w_{j}(\theta)x_j \\
&= -w_{i}(\theta)x_i
+ w_{i}(\theta)\sum_{j=1}^t w_{j}(\theta)x_j \\
&=  w_{i}(\theta)\left(\sum_{j=1}^t w_{j}(\theta)x_j-x_i\right).
\end{align*}
Now bound the $\|\cdot\|_1$-norm of the derivative:
\begin{align*}
    \left\|\frac{d}{d\theta} w(\theta)\right\|_1
=
\sum_{i=1}^t \left|\frac{d}{d\theta} w_{i}(\theta)\right|
&= \sum_{i=1}^t w_{i}(\theta)|\sum_{j=1}^t w_{j}(\theta)x_j-x_i| \\
&\leq (t-1) \sum_{i=1}^t w_{i} (\theta)\\
&\leq (t-1).
\end{align*}
Fix $\theta,\theta'\in\Theta$. By the mean value theorem, we then have
\[
\|w(\theta)-w(\theta')\|_1
 \leq (t-1)|\theta-\theta'|.
\]
Set $\eta = \epsilon/(t-1)$ and let $\{\theta_1,\dots,\theta_N\}$ be an $\eta$-net of $\Theta$
with size equal to $N_1(\eta,\Theta)$, the covering number with respect to $|\cdot|$. For any $\theta\in\Theta$ choose $\theta_k$ with $|\theta-\theta_k|\le \eta$.
This implies $\|w(\theta)-w(\theta_k)\|_1 \le (t-1)\eta = \epsilon$.
Hence $\{w(\theta_1),\dots,w(\theta_N)\}$ is an $\epsilon$-net of $\mathcal{W}_t(\Theta)$ in $\|\cdot\|_1$, and therefore $N_1(\epsilon,\mathcal{W}_t(\Theta)) \le N_1\left(\epsilon / (t-1),\Theta\right)$.
Applying Lemma~\ref{lem:covradius} with $d=1$ and radius $R$ gives
\begin{align*}
    N_1(\eta,\Theta)
\le \frac{3R}{\eta} = \frac{3R(t-1)}{\epsilon},
\end{align*}
and the final statement follows from
$$N_1(\epsilon,\bigcup_{t=1}^n\mathcal{W}_t(\Theta)) \leq \sum_{t=1}^n \frac{3R(t-1)}{\epsilon} \leq \frac{3Rn^2}{2\epsilon}.$$
\end{proof}

\begin{lemma}\label{lem:exponentialassump}
Let $\theta\in(0,\infty)$. For the exponential weights in \eqref{eq:expweights},
\[
B_{\mathcal{W}(\Theta)}\leq 2, \qquad C_1 = 1.
\]
\end{lemma}

\begin{proof}
    Define $\rho=\exp(-\theta )$. We can rewrite the exponential weights for $k=t-i\in\{0,\dots,t-1\}$ as
\begin{align*}
    w_{i}^{(t)}(\theta) = \frac{\rho^{t-i}}{\sum_{j=1}^t \rho^{t-j}} = \frac{\rho^{t-i}}{\sum_{k=0}^{t-1} \rho^{k}} = \frac{(1-\rho)\rho^{t-i}}{1-\rho^t}.
\end{align*}
    Now compute
\begin{align*}
    \|w^{(t)}(\theta)\|^2 = \sum_{i=1}^t (w_{i}^{(t)})^2(\theta) = \sum_{k=0}^{t-1} \left(\frac{(1-\rho)\rho^{k}}{1-\rho^t}\right)^2 = \frac{(1-\rho)^2}{(1-\rho^t)^2} \sum_{k=0}^{t-1} \rho^{2k} = \frac{(1-\rho)^2(1-\rho^{2t})}{(1-\rho^t)^2(1-\rho^2)}.
\end{align*}
Because the maximum is attained at $i=t$, we also have
    \begin{align*}
        \|w^{(t)}(\theta)\|_\infty = \frac{1}{\sum_{k=0}^{t-1} \rho^{k}} = \frac{1-\rho}{1-\rho^t}.
    \end{align*}
The ratio is then
\begin{align*}
     \frac{\|w^{(t)}(\theta)\|_\infty}{\|w^{(t)}(\theta)\|^2} =\frac{1-\rho}{1-\rho^t} \cdot \frac{(1-\rho^t)^2(1-\rho^2)}{(1-\rho)^2(1-\rho^{2t})} = \frac{(1-\rho^t)(1-\rho^2)}{(1-\rho)(1-\rho^{2t})} = \frac{1+\rho}{1+\rho^t} \leq 1+\rho \leq 2,
\end{align*}
which holds for any $\theta \in (0, \infty)$. The second claim follows by noting that $w^{(t)}(\theta)$ is a probability vector with nonnegative entries for any $t \in \{1,\dots,n\}$.
\end{proof}

\subsection{Brown double exponential smoothing weights}\label{app:brown-weights}

\begin{lemma}\label{lem:browndes-cover}
Fix $t\in\{1,\dots,n\}$ and let $\Theta=(0,1)$. For $\theta\in\Theta$, define the weight vector
$w^{(t)}(\theta)\in\mathbb{R}^n$ by
\begin{equation}\label{eq:browndes-weights}
    w_{i}^{(t)}(\theta)
    =
     \mathbbm{1}\{i \leq t\}\frac{\theta[2-\theta(t-i+1)](1-\theta)^{t-i}}{\sum_{j=1}^t \theta[2-\theta(t-j+1)](1-\theta)^{t-j}}.
\end{equation}
Let $\mathcal{W}_t(\Theta)=\{w^{(t)}(\theta):\theta\in\Theta\}$. Then for every $\epsilon>0$,
\[
N_1(\epsilon,\mathcal{W}_t(\Theta)) \le \frac{60t}{\epsilon}, \qquad N_1\left(\epsilon, \bigcup_{t = 1}^n \mathcal{W}_t(\Theta)\right)
\le
\frac{60n^2}{\epsilon}.
\]
\end{lemma}

\begin{proof}
Fix $t \in \{1, \dots, n\}$, $\theta\in(0,1)$ and write $k=t-i\in\{0,\dots,t-1\}$ and $r=1-\theta\in(0,1)$.
We again abbreviate $w(\theta) = w^{(t)}(\theta)$.
Observe that
\begin{align*}
 w_{i}(\theta) = w_{t-k}(\theta)=  \frac{\theta[2-\theta(k+1)]r^k}{\sum_{j=0}^{t-1} \theta[2-\theta(j+1)]r^j} =  \frac{[2-\theta(k+1)]r^k}{\sum_{j=0}^{t-1} [2-\theta(j+1)]r^j}= \frac{b_k(\theta)}{\sum_{j=0}^{t-1}b_j(\theta)}.
\end{align*}
Differentiating $w_{t-k}(\theta)$ gives
\begin{align*}
 \frac{d}{d\theta}w_{t-k}(\theta) =\frac{b_k'(\theta)}{\sum_{j=0}^{t-1}b_j(\theta)}-\frac{b_k(\theta)(\sum_{j=0}^{t-1} b_j'(\theta))}{(\sum_{j=0}^{t-1} b_j(\theta))^2}.
\end{align*}
Therefore, by the triangle inequality,
\begin{align*}
\left\|\frac{d}{d\theta}w(\theta)\right\|_1
&\le \frac{\sum_{k=0}^{t-1}|b_k'(\theta)|}{\sum_{j=0}^{t-1}b_j(\theta)}
+\frac{|\sum_{j=0}^{t-1} b_j'(\theta)|}{(\sum_{j=0}^{t-1}b_j(\theta))^2}\sum_{k=0}^{t-1}|b_k(\theta)|
\nonumber\\
&\le \frac{\sum_{k=0}^{t-1}|b_k'(\theta)|}{\sum_{j=0}^{t-1}b_j(\theta)}
\left(1+\frac{\sum_{k=0}^{t-1}|b_k(\theta)|}{\sum_{j=0}^{t-1}b_j(\theta)}\right).
\end{align*}
We start by bounding $\sum_{k=0}^{t-1} |b_k|$. Using geometric-sum identities, we obtain
\begin{align*}
\sum_{j=0}^{t-1}b_j(\theta)
&=\sum_{j=0}^{t-1}\big(2-\theta(j+1)\big)r^j
=2\sum_{j=0}^{t-1}r^j-\theta\sum_{j=0}^{t-1}(j+1)r^j\\
&=\frac{2(1-r^t)}{\theta}-\frac{1-(t+1)r^t+tr^{t+1}}{\theta}
=\frac{1+r^t(t\theta-1)}{\theta}.
\end{align*}
Hence $\sum_{j=0}^{t-1}b_j(\theta)>0$ for all $\theta\in(0,1)$ and, in particular,
\begin{align*}
\sum_{j=0}^{t-1}b_j(\theta)\ge
\begin{cases}
t, & \theta\le 1/t,\\
1/\theta, & \theta>1/t.
\end{cases}
\end{align*}
Indeed, if $\theta\le 1/t$ then $t\theta-1\le 0$, $r^t(t\theta-1) \geq (t\theta-1)$ and thus $1+r^t(t\theta-1)\ge t\theta$, so $\sum_{j=0}^{t-1}b_j(\theta)\ge t$.
If $\theta>1/t$ then $t\theta-1>0$ and $1+r^t(t\theta-1)\ge 1$, so $\sum_{j=0}^{t-1}b_j(\theta)\ge 1/\theta$. Using $|2-\theta(k+1)|\le 2+\theta(k+1)$,
\[
|b_k(\theta)|\le (2+\theta(k+1))r^k = 2r^k+\theta(k+1)r^k.
\]
Summing and applying the same series arguments as before,
\[
\sum_{k=0}^{t-1}|b_k(\theta)|
\le 2\sum_{k=0}^{t-1}r^k+\theta\sum_{k=0}^{t-1}(k+1)r^k
\le \frac{3}{\theta}.
\]
We now bound $\sum_{k=0}^{t-1} |b_k'|/\sum_{j=0}^{t-1}b_j(\theta)$. First, compute
\begin{align*}
\frac{d}{d\theta}((2-\theta(k+1))r^k) =-(k+1)r^k-k(2-\theta(k+1))r^{k-1}=b_k'(\theta).
\end{align*}
From the expression for $b_k'(\theta)$ and $|2-\theta(k+1)|\le 2+\theta(k+1)$,
\[
|b_k'(\theta)|
\le (k+1)r^k + k(2+\theta(k+1))r^{k-1}
\le (k+1)r^k + 2k r^{k-1} + \theta k(k+1)r^{k-1}.
\]
Hence
\[
\sum_{k=0}^{t-1}|b_k'(\theta)|
\le \sum_{k=0}^{t-1}(k+1)r^k + 2\sum_{k=0}^{t-1}k r^{k-1} + \theta\sum_{k=0}^{t-1}k(k+1)r^{k-1}.
\]
We bound the sums in the two regimes.\\
\emph{Case 1: $\theta\le 1/t$.}
Use the finite-sum bounds $r^{k}\le 1$ and $r^{k-1}\le 1$. Since $\theta\le 1/t$, we obtain
\begin{align*}
    \sum_{k=0}^{t-1}|b_k'(\theta)|& \leq \sum_{k=0}^{t-1}(k+1)+2\sum_{k=0}^{t-1}k+\theta\sum_{k=0}^{t-1}k(k+1) \\
&\le \frac{t(t+1)}{2}+2\cdot\frac{t(t-1)}{2}+\theta t^3 \\
&\le \frac{t(t+1)}{2}+t(t-1)+t^2 \\
&\le 3t^2.
\end{align*}
Moreover, $\sum_{j=0}^{t-1}b_j(\theta)\ge t$, hence
\begin{align*}
\frac{\sum_{k=0}^{t-1}|b_k'(\theta)|}{\sum_{j=0}^{t-1}b_j(\theta)} \le \frac{3t^2}{t}=3t.
\end{align*}
We also have $2-\theta(k+1)\ge 1$ for all $k\le t-1$, so $b_k(\theta)\ge 0$ and therefore
$\sum_{k=0}^{t-1}|b_k(\theta)|=\sum_{k=0}^{t-1}b_k(\theta)$, i.e., $\|w^{(t)}(\theta)\|_1=\sum_{k=0}^{t-1} |b_k|/\sum_{j=0}^{t-1}b_j(\theta)=1$. \\
\emph{Case 2: $\theta>1/t$.}
Extend the sums to infinity and use standard identities:
\begin{align*}
    \sum_{k=0}^{t-1}|b_k'(\theta)| &\le \sum_{k=0}^{\infty}(k+1)r^k+ 2\sum_{k=0}^{\infty}k r^{k-1}+ \theta\sum_{k=0}^{\infty}k(k+1) r^{k-1} \\
    &=\frac{1}{(1-r)^2}+\frac{2}{(1-r)^2} +\frac{2\theta}{(1-r)^3}
\\ &=\frac{1}{\theta^2}+\frac{2}{\theta^2}+\frac{2\theta}{\theta^3} \\
&=\frac{5}{\theta^2}.
\end{align*}
Since $\sum_{j=0}^{t-1}b_j(\theta)\ge 1/\theta$ in this regime,
\begin{align*}
\frac{\sum_{k=0}^{t-1}|b_k'(\theta)|}{\sum_{j=0}^{t-1}b_j(\theta)} \le \frac{5/\theta^2}{1/\theta}=\frac{5}{\theta}\le 5t,
\end{align*}
and it further holds $\|w^{(t)}(\theta)\|_1=\sum_{k=0}^{t-1} |b_k(\theta)|/\sum_{j=0}^{t-1}b_j(\theta)\leq 3$ because weights may be signed.  \\
The bound from the second case, where $\theta>1/t$ and the weights may be negative, is larger than the bound from the first case, where the weights lie on a probability simplex. Therefore, for any $\theta \in (0,1)$, we obtain
\[
\left\|\frac{d}{d\theta}w^{(t)}(\theta)\right\|_1
\le (5t)(1+3)=20t.
\]
Fix $\theta,\theta'\in\Theta$. By the mean value theorem, $\|w^{(t)}(\theta)-w^{(t)}(\theta')\|_1 \le 20t\,|\theta-\theta'|$. Set $\eta=\epsilon/(20t)$ and let $\{\theta_1,\dots,\theta_N\}$ be an $\eta$-net of $\Theta$.
Arguing as in Lemma~\ref{lem:expweights-cover}, we obtain
$N_1(\epsilon,\mathcal{W}_t(\Theta))\le N_1(\eta,\Theta)$.
Applying Lemma~\ref{lem:covradius} with $d=1$ and radius $R=1$ yields $N_1(\eta,\Theta)\le 3  / \eta = 60t / \epsilon$, and the last statement follows immediately from bounding the covering number of the union over $t=1,\dots,n$.
\end{proof}

\begin{lemma}\label{lem:browndes-assum}
Let $\Theta=(0,1)$ and $\mathcal{W}_t(\Theta)=\{w^{(t)}(\theta):\theta\in\Theta\}$.
For weights of the form \eqref{eq:browndes-weights},
\[
B_{\mathcal{W}}\le 18e^2,
\qquad
C_1\leq 3.
\]
\end{lemma}

\begin{proof}
We will use the same notation as in the preceding proof. Therefore, let $r=1-\theta\in(0,1)$ and
$b_k(\theta)=(2-\theta(k+1))r^k$ for $k=0,\dots,t-1$. Moreover, we again have to consider two regimes for $\theta$. \\
\emph{Case 1: $\theta \leq 1/t$.} For a fixed $\theta$ and all $k\le t-1$, we have $2-\theta(k+1)\ge 2-\theta t\ge 1$, which implies $b_k(\theta)\ge r^k\ge r^{t-1}.$ Moreover, since $\theta\le 1/t$, we have $r\ge 1-1/t$ and thus $b_k(\theta) \geq r^{t-1}\ge \Bigl(1-\frac{1}{t}\Bigr)^{t-1}\ge e^{-1}$
for all $k$. Therefore
\[
\sum_{k=0}^{t-1} b_k(\theta)^2 \ge \sum_{k=0}^{t-1} e^{-2} = t e^{-2}.
\]
Furthermore, $b_0(\theta)=2-\theta\le 2$, so $\max_{0\le k\le t-1}|b_k(\theta)|\le 2$, and trivially $\sum_{k=0}^{t-1} b_k(\theta)\le \sum_{k=0}^{t-1} 2 = 2t$. Therefore
\[
\frac{\|w^{(t)}(\theta)\|_\infty}{\|w^{(t)}(\theta)\|^2}
=
\frac{(\sum_{k=0}^{t-1}b_k(\theta))\max_{0\le k\le t-1}|b_k(\theta)|}{\sum_{k=0}^{t-1} b_k(\theta)^2} \le
\frac{(2t)\cdot 2}{t e^{-2}}
=4e^{2}.
\]
\emph{Case 2: $\theta> 1/t$.}
Fix $\theta>1/t$ and set $m=\lfloor 1/(2\theta)\rfloor$. Then $m\le t-1$, so the indices
$k=0,\dots,m$ are available. For $k\le m$ we have $\theta(k+1)\le \theta(m+1)\le 1
\Rightarrow
2-\theta(k+1)\ge 1$, and also $r^k\ge r^m$. Since $\log(1-\theta)\ge -2\theta$ for $\theta\in(0,1/2]$, we obtain
\[
r^m=(1-\theta)^m \ge (1-\theta)^{1/(2\theta)} \ge e^{-1},
\]
whenever $\theta\in(1/t,1/2]$. Hence $b_k(\theta)\ge e^{-1}$ for all $k\le m$ in this regime, and therefore
\[
\sum_{k=0}^{t-1} b_k(\theta)^2 \ge \sum_{k=0}^{m} e^{-2} = (m+1)e^{-2}.
\]
Since $m=\lfloor 1/(2\theta)\rfloor$, we have $\lfloor x\rfloor+1\ge x$ for all $x\ge 0$, and thus
$m+1=\lfloor \frac{1}{2\theta}\rfloor+1 \ge \frac{1}{2\theta}.$
Consequently, for $\theta\in(1/t,1/2]$,
\[
\sum_{k=0}^{t-1} b_k(\theta)^2 \ge \frac{e^{-2}}{2\theta}.
\]
For $\theta\in(1/2,1)$ we instead use the trivial bound
$\sum_{k=0}^{t-1} b_k(\theta)^2\ge b_0(\theta)^2=(2-\theta)^2\ge 1$.

Moreover, the previous lemma shows that, for all $\theta\in(0,1)$,
\[
\sum_{k=0}^{t-1} b_k(\theta)\le \sum_{k=0}^{t-1}|b_k(\theta)|\le \frac{3}{\theta}.
\]
Finally, for all $k\ge 0$,
$|b_k(\theta)|\le (2+\theta(k+1))r^k \le 2 + \theta(k+1)r^k$.
Using $r^k=(1-\theta)^k\le e^{-\theta k}$ and maximizing the function
$x\mapsto \theta(x+1)e^{-\theta x}$ over $x\ge 0$ yields the maximizer $x^*=1/\theta-1$ and
\[
\sup_{x\ge 0}\theta(x+1)e^{-\theta x}
=\theta\left(\frac{1}{\theta}\right)e^{-\theta(1/\theta-1)}
=e^{-1+\theta}\le 1.
\]
Therefore, $\max_{0\le k\le t-1}|b_k(\theta)| \le 2+1=3$, and we obtain the uniform bound
\[
\frac{\|w^{(t)}(\theta)\|_\infty}{\|w^{(t)}(\theta)\|^2}
=
\frac{|\sum_{k=0}^{t-1}b_k(\theta)|\max_{0\le k\le t-1}|b_k(\theta)|}{\sum_{k=0}^{t-1} b_k(\theta)^2}
\le
\begin{cases}
\displaystyle \frac{\frac{3}{\theta}\cdot 3}{\frac{e^{-2}}{2\theta}} = 18e^2, & \theta\in(1/t,1/2],\\[10pt]
\displaystyle \frac{\frac{3}{\theta}\cdot 3}{1}\le 18, & \theta\in(1/2,1).
\end{cases}
\]
The proof of the preceding lemma gives $\sup_{\theta} \|w^{(t)}(\theta)\|_1 \leq 3$ for any $t$, yielding the second claim.
\end{proof}

\end{document}